\documentclass[runningheads]{llncs}

\usepackage{eccv}
\usepackage{eccvabbrv}
\usepackage{graphicx}
\usepackage{booktabs}
\usepackage[accsupp]{axessibility}

\newcommand{\Sref}[1]{Section~\ref{#1}}

\newcommand{\x}{\mathbf{x}}

\newcommand{\z}{\mathbf{z}}

\usepackage{algorithm}
\usepackage{algorithmic}

\usepackage{array}
\usepackage{multirow}
\usepackage[table]{xcolor}
\definecolor{nullbg}{gray}{0.93}
\newcolumntype{g}{>{\columncolor{nullbg}}c}

\usepackage{hyperref}
\usepackage{orcidlink}

\begin{document}

\title{Retrieving and Refining Winning Noise Tickets for Diffusion-Based Motion Generation}
\titlerunning{Winning Noise Tickets for Motion Generation}

\author{Sakuya Ota\inst{1}\orcidlink{0009-0008-8070-7263} \and
Qing Yu\inst{2}\orcidlink{0000-0001-6965-9581} \and
Kent Fujiwara\inst{2}\orcidlink{0000-0002-2205-6115} \and \\
Satoshi Ikehata\inst{3}\orcidlink{0000-0002-6061-7956} \and
Ikuro Sato\inst{1}\orcidlink{0000-0001-5234-3177}
}

\authorrunning{S. Ota et al.}

\institute{Institute of Science Tokyo, Tokyo, Japan \and
LY Corporation, Tokyo, Japan \and
National Institute of Informatics (NII), Tokyo, Japan
\email{ota\_sakuya@d-itlab.c.titech.ac.jp}
}

\maketitle

\begin{abstract}
Diffusion-based text-to-motion models synthesize realistic human motions but often exhibit semantic drift from the input text. Motion is inherently temporal, especially in compositional and long-duration sequences that require semantic consistency across multiple action segments and smooth kinematic transitions throughout the trajectory. We posit that the initial noise is central to this consistency: within the Gaussian noise space, certain instances, \ie \textit{winning noise tickets}, carry latent structure that biases denoising toward particular motion semantics, even under null prompts. We propose WInning Noise Retrieval and Optimization (WINRO), a training-free, model-agnostic framework that improves text--motion alignment by selecting and refining such tickets before diffusion sampling. WINRO maps random noises to motion features generated under null prompts, retrieves the best-aligned noise for a given text, and refines it via a KL-regularized objective that reduces the residual semantic gap while preserving the Gaussian prior. An optional LoRA-based adapter amortizes this refinement into a single forward pass. WINRO consistently improves text--motion fidelity across different base models, MDM and MotionLCM, on HumanML3D without retraining, improves temporal robustness on the MTT benchmark, and generalizes to applications such as motion stylization and spatial constraint satisfaction.
\keywords{Human motion generation \and Diffusion models \and Text-to-motion \and Noise optimization}
\end{abstract}

\section{Introduction}

Human motion generation is a fundamental task in computer graphics, robotics, and virtual reality, where the goal is to synthesize realistic, semantically coherent motion sequences from high-level inputs. In particular,
text-to-motion (T2M) diffusion models translate natural language descriptions into motion trajectories, enabling intuitive workflows for animation and embodied AI systems~\cite{athanasiou2022teach, chen2023executing, ghosh2021synthesis, kalakonda2022action, tevet2023human, zhang2023t2m, Plappert2016kit}.
Despite their success, diffusion-based T2M models often suffer from semantic drift, a challenge particularly pronounced in temporally structured, long-horizon generation, where the text specifies multiple action segments and the generated trajectory must realize them with consistent semantics and smooth kinematic transitions. This limits their applicability in scenarios demanding fine-grained control over style or body-part behavior.

\begin{figure}[t]
\centering
\includegraphics[width=1\linewidth]{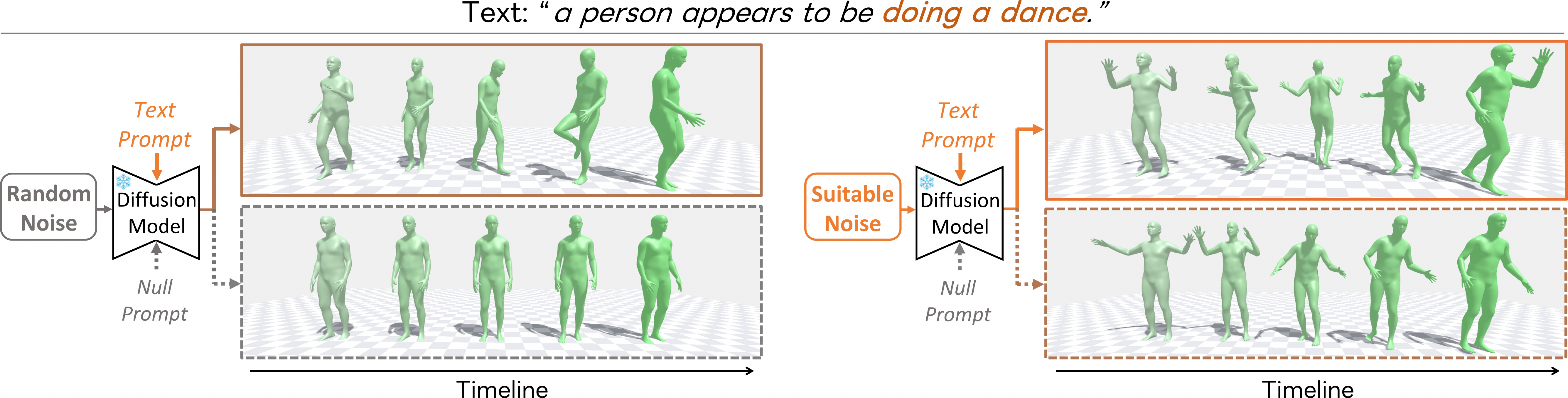}
\caption{Winning Noise Ticket Hypothesis for motion. Certain initial noises inherently encode motion semantics, producing meaningful motion even under a null prompt (right), while others yield static poses regardless of conditioning (left). WINRO retrieves and refines such tickets to improve text--motion alignment.}
\vspace{-15pt}
\label{fig:nullandtext}
\end{figure}

We posit that within the Gaussian noise space used to initialize diffusion-based motion generators, there exist specific noise instances, termed \textit{winning noise tickets}, that consistently bias the denoising trajectory toward temporally coherent motions with particular semantics or styles. This notion builds on the extension of the Lottery Ticket Hypothesis from network pruning~\cite{frankle2018lottery} to diffusion models~\cite{mao2024lottery}. Empirically, certain noise samples produce coherent, meaningful motions even under null prompts, while others yield implausible or text-misaligned results even when conditioned on descriptive text (Fig.~\ref{fig:nullandtext}). Thus, the initial noise is not mere randomness, but acts as a latent semantic prior that shapes the generative trajectory.
While similar observations have been made for image diffusion~\cite{mao2024lottery}, the effect is especially impactful in motion generation, where training data is orders of magnitude scarcer (approximately 15k samples vs.\ billions of images), making models more sensitive to the choice of initial noise.

Prior works that optimize initial noise in motion diffusion do so to satisfy physical or geometric constraints~\cite{karunratanakul2024dno, liu2024programmable}. Because their objectives are defined over predefined spatial targets rather than the text prompt itself, the optimized noise may fulfill the physical goal while still producing motion that drifts from the intended textual meaning. We address this gap by directly optimizing noise against text--motion semantic similarity, and ask:
\textit{Can we identify and refine the winning noise tickets that best embody the semantics of a given text prompt?}

To address this, we introduce \textbf{WI}nning \textbf{N}oise \textbf{R}etrieval and \textbf{O}ptimization (\textbf{WINRO}), a \textit{training-free and model-agnostic} framework that improves text--motion alignment by discovering and refining winning noise tickets before the diffusion process.
Selecting a semantically aligned noise requires associating each noise sample with a descriptor that (i) captures motion semantics and (ii) is comparable with text. We achieve this by constructing a \textit{noise dictionary}. We sample a random Gaussian noise and use a frozen diffusion model to generate a motion under a null prompt. A pretrained text--motion retrieval model then encodes the motion into a feature vector in a shared text--motion embedding space, which we store with the noise as a dictionary entry.
Given an input text prompt, WINRO retrieves the noise whose motion feature is closest to the text embedding in this space, providing a semantically favorable initialization.
The retrieved noise is then refined through a \textit{KL-regularized objective} that minimizes the residual text–motion discrepancy while preserving the Gaussian prior. Optionally, a trained LoRA-based Noise Refiner can amortize this iterative refinement into a single forward pass.
The key contributions are:
\begin{itemize}
\setlength\itemsep{0pt}
\item We identify the \textit{Winning Noise Ticket} in motion diffusion, showing that certain initial noises inherently bias denoising trajectories toward specific motion semantics.
\item We propose WINRO, a training-free framework for \emph{motion-specific} semantic noise retrieval and refinement, featuring a null-prompt motion dictionary, frame-adjusted retrieval across variable durations, and KL-regularized optimization in the text--motion embedding space.
\item We optionally extend WINRO with a LoRA-based Noise Refiner that, once trained, amortizes test-time optimization into a single forward pass without updating the base model weights.
\item We demonstrate consistent improvements in text--motion fidelity across diffusion backbones on HumanML3D, improved long-horizon temporal robustness on the MTT benchmark, and generalization to motion stylization and spatial constraint satisfaction.
\end{itemize}

\section{Related Work}
\label{sec:relatedwork}

\paragraph{Motion Generation Models.}
Advances in motion-capture and video-based reconstruction~\cite{kocabas2019vibe, ye2023slahmr, whamcvpr2024} have driven interest in generating realistic human motion.
Early methods focused on motion prediction or interpolation~\cite{holden2020learned, yuan2020dlow}, later extending to action-conditioned synthesis~\cite{Petrovich_2021_ICCV, Xu_2023_ICCV}.
More recently, \textit{text-to-motion (T2M)} diffusion models~\cite{ghosh2021synthesis, Guo_2022_CVPR_humanml3d, zhang2023t2m, jiang2023motiongpt} have emerged as powerful tools for natural-language-driven motion generation.
Notably, MotionDiffuse~\cite{zhang2022motiondiffuse}, MDM~\cite{tevet2023human}, MotionLCM~\cite{motionlcm}, and the skeleton-aware latent diffusion model SALAD~\cite{hong2025salad} achieve strong realism but treat initial noise as an unstructured random seed, leaving its semantic role unexplored.

\paragraph{Motion-Text Retrieval.}
Motion-text retrieval models~\cite{petrovich2023tmr, yu2024exploring, Fujiwara_2024_ECCV} learn aligned multimodal spaces for semantic comparison.
Retrieval-augmented generation (RAG) frameworks such as ReMoDiffuse~\cite{zhang2023remodiffuse} retrieve motion exemplars to condition the denoising process, improving fidelity through external references~\cite{yu2025remogpt, li2025remomask}.
Rather than retrieving exemplar motions, our method retrieves semantically meaningful initial noises whose induced null-prompt motions align with the input text, making it complementary to exemplar-based retrieval.

\paragraph{Initial Noise Selection and Optimization.}
Recent studies show that the initial noise in diffusion models affects generated outcomes~\cite{mao2023guided, xu2024good}, motivating noise optimization~\cite{chen2024find, eyring2024reno, guo2024initno, eyring2025noise} and selection~\cite{wang2024silent, mao2024lottery}.
The \textit{Lottery Ticket Hypothesis in Denoising}~\cite{mao2024lottery} suggests that certain ``winning'' noise patterns bias denoising toward specific semantics.
In motion generation, such biases are consequential, as the initial noise can affect temporal semantic consistency.
Accordingly, in the motion domain, works such as DNO~\cite{karunratanakul2024dno}, ProgMoGen~\cite{liu2024programmable}, and PINO~\cite{ota2025pino} optimize noise via physical or geometric constraints.
In this research, we target text--motion semantic alignment, retrieving noise from a precomputed dictionary and refining it against the text prompt.

\section{Winning Noise Tickets for Motion Diffusion}
\label{sec:winning_noise}

\subsection{Diffusion-Based Motion Generation}
We build upon diffusion-based text-to-motion (T2M) models, which generate motion by progressively denoising Gaussian noise.
A motion sequence $\x \in \mathbb{R}^{F \times D}$ has $F$ frames and $D{=}263$ features~\cite{Guo_2022_CVPR_humanml3d}.
The forward process gradually corrupts clean motion $\x_0$ into noise $\x_T$ over $T$ steps:
\begin{equation}
q(\x_t \mid \x_{t-1}) = \mathcal{N}(\x_t; \sqrt{\alpha_t}\x_{t-1}, (1-\alpha_t)\mathbf{I})\,,
\end{equation}
where $\alpha_t$ is the noise schedule. With a sufficiently long schedule, we approximate $\x_T \!\sim\! \mathcal{N}(\mathbf{0}, \mathbf{I})$.
The denoising network $G_\theta$ learns to invert this process:
\begin{equation}
\x_0 = G_\theta(\x_T, c) = [f_0 \circ f_1 \cdots \circ f_T](\x_T, c)\,,
\label{eq:generation}
\end{equation}
where $f_t$ is the $t$-th denoising update and $c$ is a conditioning input such as text.
Models like MDM~\cite{tevet2023human} operate in the motion domain, while latent variants (\eg, MLD~\cite{chen2023executing}) denoise compact motion latents.
With deterministic samplers (\eg, DDIM~\cite{song2021denoising}), there is a one-to-one mapping between the initial noise $\x_T$ and the generated motion $\x_0$, making the initial noise a controllable factor.

\subsection{Theoretical Perspective: Why Noise Encodes Motion Semantics}
\label{sec:theory_noise}
Under deterministic sampling, diffusion defines a continuous map from initial noise to a motion trajectory, so nearby noises tend to produce similar motions.
Composed with a retrieval encoder, the resulting text--motion alignment score varies smoothly over the noise space; hence high-alignment initial noise form \emph{regions} rather than isolated points, enabling both retrieval and local refinement.

\paragraph{Setup.}
Let $\mathcal{F}_\text{Motion}$ and $\mathcal{F}_\text{Text}$ be the motion--text retrieval encoders.
For deterministic sampling, define the cosine alignment score
\begin{equation}
S_c(\x_T)=
\frac{
\mathcal{F}_\text{Motion}(G_\theta(\x_T,c))^\top \mathcal{F}_\text{Text}(c)
}{
\lVert \mathcal{F}_\text{Motion}(G_\theta(\x_T,c))\rVert\,
\lVert \mathcal{F}_\text{Text}(c)\rVert
}\in[-1,1]\,.
\end{equation}
We use the \emph{null} score $\tilde{S}_c(\x_T)$ to describe intrinsic semantics under unconditional generation, defined by setting the diffusion condition to $c{=}\varnothing$ while keeping the query text embedding $\mathcal{F}_\text{Text}(c)$. We call $\x_T$ a \emph{winning noise ticket} for $c$ (level $\gamma$) if $\tilde{S}_c(\x_T)\ge\gamma$.

\paragraph{Smoothness implies semantic regions.}
If the deterministic sampler and $\mathcal{F}_\text{Motion}$ are locally Lipschitz, then $\tilde{S}_c$ is locally Lipschitz in $\x_T$.
Therefore, whenever $\tilde{S}_c(\x_T^\star)\ge \gamma$, there exists a radius $r>0$ such that all $\x_T$ with $\lVert \x_T-\x_T^\star\rVert\le r$ also satisfy $\tilde{S}_c(\x_T)\ge \gamma/2$.
This locality motivates (i) best-of-$N$ dictionary retrieval and (ii) small, prior-preserving refinements (Sec.~\ref{sec:method}).

\subsection{Evidence of Winning Noise Tickets}
To examine whether initial noise influences motion semantics, we analyze T2M generation using MDM~\cite{tevet2023human} and a text–motion retrieval model TMR~\cite{petrovich2023tmr}, measuring text--motion similarity as cosine similarity in the shared embedding space.

\begin{figure}[t]
  \centering
  \begin{subfigure}[t]{0.33\linewidth}
    \centering
    \includegraphics[width=\linewidth,trim=2pt 2pt 2pt 2pt,clip]{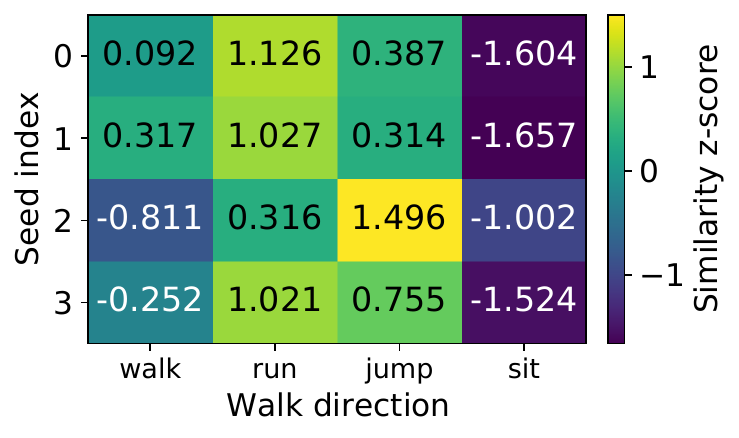}
    \subcaption{Action category}
    \label{fig:noise-a}
  \end{subfigure}\hfill
  \begin{subfigure}[t]{0.33\linewidth}
    \centering
    \includegraphics[width=\linewidth,trim=2pt 2pt 2pt 2pt,clip]{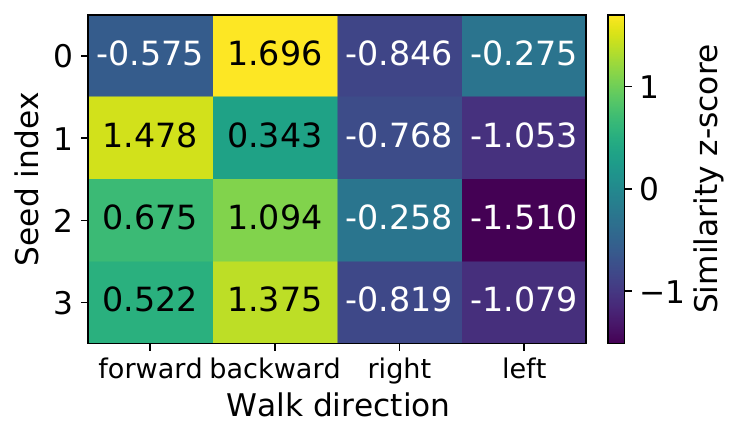}
    \subcaption{Action direction}
    \label{fig:noise-a-dir}
  \end{subfigure}
  \begin{subfigure}[t]{0.33\linewidth}
    \centering
    \includegraphics[width=\linewidth,trim=2pt 2pt 2pt 2pt,clip]{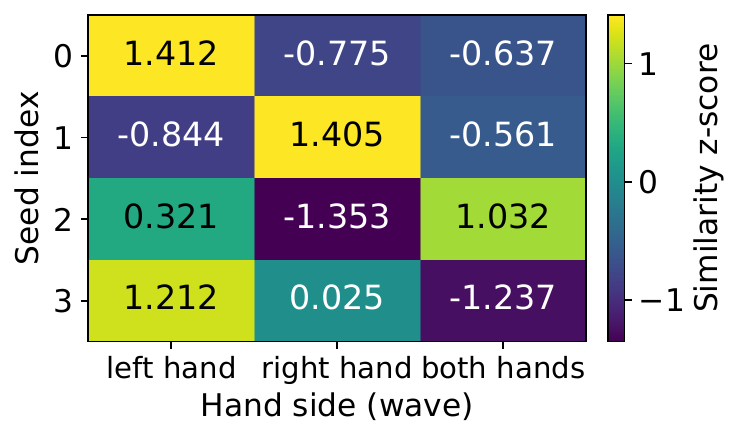}
    \subcaption{Action parts}
    \label{fig:noise-a-parts}
  \end{subfigure}
  \caption{$z$-score heatmaps of action-specific tendencies under four fixed noise seeds. For each seed and prompt group, we generate ten motions with varied prompts, compute the mean TMR similarity per category, and apply per-seed $z$-score normalization.}

  \label{fig:noise-analysis-fixseed}
\end{figure}

\begin{figure}[t]
  \centering
  \captionsetup{skip=3pt}
  \begin{subfigure}[t]{0.33\linewidth}
    \centering
    \includegraphics[width=\linewidth,trim=2pt 2pt 2pt 2pt,clip]{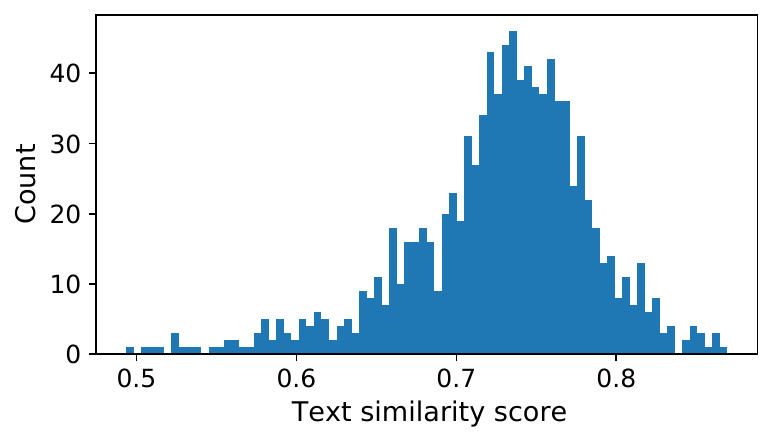}
    \subcaption{``walk''}
    \label{fig:noise-b-walk}
  \end{subfigure}\hfill
  \begin{subfigure}[t]{0.33\linewidth}
    \centering
    \includegraphics[width=\linewidth,trim=2pt 2pt 2pt 2pt,clip]{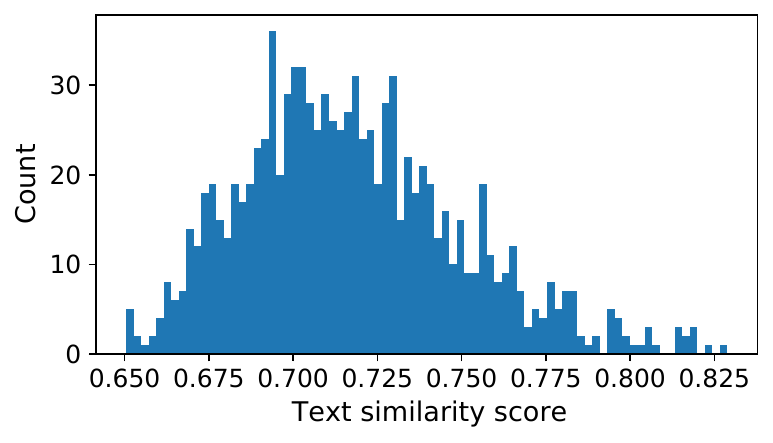}
    \subcaption{``run''}
    \label{fig:noise-b-run}
  \end{subfigure}
  \begin{subfigure}[t]{0.33\linewidth}
    \centering
    \includegraphics[width=\linewidth,trim=2pt 2pt 2pt 2pt,clip]{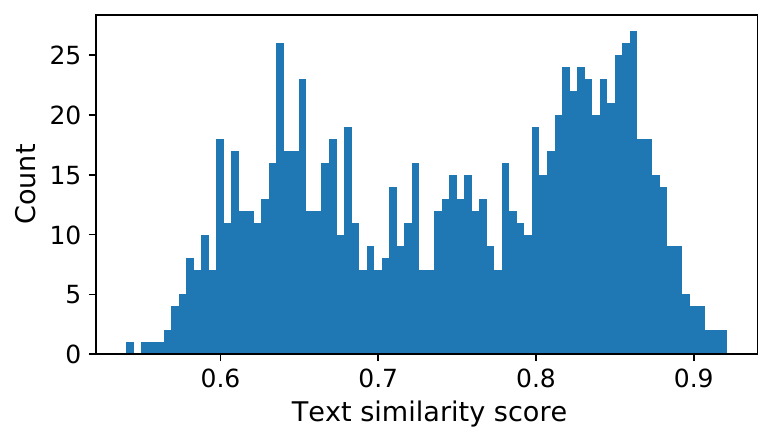}
    \subcaption{``jump''}
    \label{fig:noise-b-jump}
  \end{subfigure}
  \caption{Similarity-score histograms for the fixed prompts ``the person walks/runs/jumps.''
For each prompt, we compute TMR similarity scores over 1,000 initial noise seeds.}
  \label{fig:noise-analysis-fixprompt}
\end{figure}

\paragraph{Bias Patterns Under Fixed Noise.}
We fix four random noise seeds and measure text–motion similarity across prompt groups targeting action categories, directions, and body parts. Fig.~\ref{fig:noise-analysis-fixseed} shows clear seed-specific preferences, especially for direction and body-part movement (Fig.~\ref{fig:noise-a-dir},~\ref{fig:noise-a-parts}), indicating distinct semantic regions in the noise space.

\paragraph{Alignment Variability Under Noise Variation.}
Conversely, fixing three prompts and varying the noise over 1{,}000 seeds reveals prompt-dependent sensitivity (Fig.~\ref{fig:noise-analysis-fixprompt}): some prompts have only a small subset of highly compatible seeds, while others show broader variability. This confirms that prompt-specific winning noise tickets exist, motivating retrieval and refinement.

We build on these observations in Sec.~\ref{sec:method} by retrieving promising seeds from a null-prompt dictionary (Sec.~\ref{sec:stage_1}) and refining them with a KL-regularized objective that preserves the Gaussian prior (Sec.~\ref{sec:stage_2}).

\section{Winning Noise Retrieval and Optimization}
\label{sec:method}

\begin{figure}[t]
    \centering
    \includegraphics[width=1\linewidth]{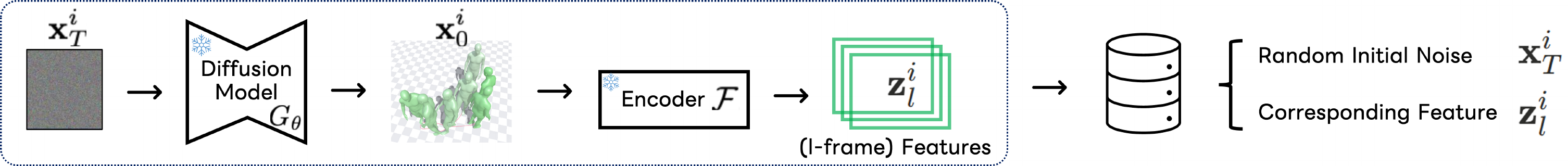}
    \caption{The pipeline of building noise dictionary $\mathcal{D}$. We randomly sample noise $\x^i_{T}$ and use a diffusion model $G_{\theta}$ and encoder $\mathcal{F}$ of a retrieval model to generate a feature $\mathbf{z}^i_{l}$. For the frame-adjusted dictionary, we truncate the output motion $\x^i_{0}$ at frame $l$ and generate the features.}
    \label{fig:dictionary}
\end{figure}

We propose to exploit this semantic bias of initial noise with two core stages: (i) retrieving a semantically favorable noise seed from a \textit{null-prompt noise dictionary} built in a text--motion embedding space, yielding WIN-FAR, a training-free and real-time variant, and (ii) optionally refining that seed by optimizing a KL-regularized text--motion similarity objective, yielding the full WINRO pipeline, which remains training-free but is iterative and therefore suited for offline or quality-critical use.
Under deterministic sampling, the alignment varies smoothly with $\x_T$ (Sec.~\ref{sec:theory_noise}), so retrieval provides a good initialization for local refinement.
Both stages are training-free and model-agnostic.

\subsection{Stage 1: Winning Noise Ticket Retrieval}
\label{sec:stage_1}
We build a noise dictionary $\mathcal{D}$ using a frozen diffusion model $G_\theta$ and a pretrained motion--text retrieval model (Fig.~\ref{fig:dictionary}).
This construction is a one-time offline cost for each backbone, after which the dictionary is reused across all queries. Per-query inference requires only a retrieval lookup and optional refinement, without additional diffusion runs for dictionary construction.
For each entry $i\in[1,\cdots,N]$, we sample $\x^i_T\!\sim\!\mathcal{N}(\mathbf{0},\mathbf{I})$, generate a null-prompt motion $\x^i_0=G_\theta(\x^i_T,\varnothing)$, and store its retrieval feature $\z^i_l=\mathcal{F}_\text{Motion}(\x^i_0)$:
\begin{equation}
\mathcal{D}_{reg}=\{(\z^i_l,\x^i_T)\}_{i=1}^N\,.
\end{equation}
To support retrieval across varying target durations, we introduce frame-adjusted retrieval (FAR). Specifically, we encode truncated motions at multiple lengths $l'\!\in[m,M]$ and store their length-specific retrieval features:
\begin{equation}
\mathcal{D}_{adj}=\{(\z^i_{l'},\x^i_T)\mid i\in[1,N],\,l'\in[m,M]\}\,.
\end{equation}

\paragraph{Why retrieval works.}
Let $\tilde{S}_c(\x_T)$ denote the null-prompt alignment score used for retrieval (Sec.~\ref{sec:theory_noise}).
If $p_{c,\gamma}=\mathbb{P}(\tilde{S}_c(\x_T)\ge\gamma)$, then for $N$ independent samples, $\mathbb{P}(\max_{i\in[N]}\tilde{S}_c(\x_T^{(i)})\ge\gamma)=1-(1-p_{c,\gamma})^N$. Thus, the probability of sampling a winning region increases with dictionary size $N$, but with diminishing returns.
Moreover, null-prompt retrieval can transfer to text prompts when conditioning acts as a limited perturbation in the retrieval space, consistent with classifier-free guidance mixing unconditional and conditional predictions. The detailed proof is provided in the supplementary materials.

\begin{figure}[t]
    \centering
    \includegraphics[width=1\linewidth]{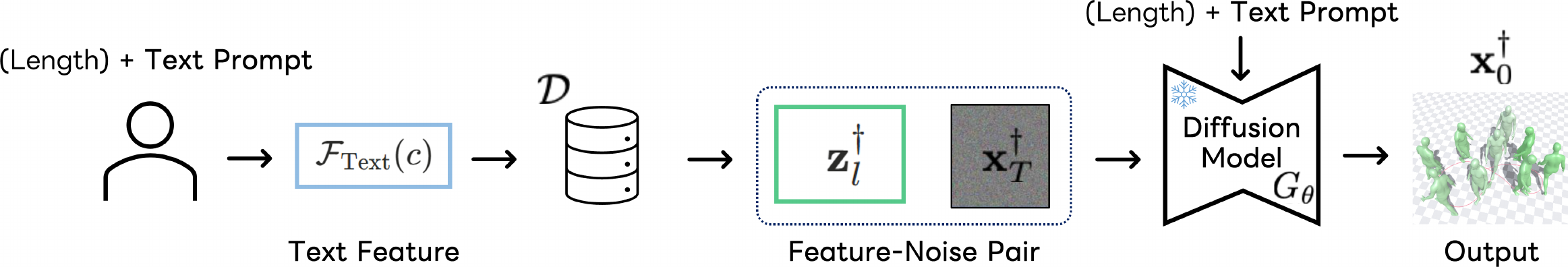}
    \caption{Overview of Winning Noise Ticket Retrieval. The input text $c$ and, optionally, the desired motion length $l$ are used to retrieve the closest feature $\z^{\dagger}_l$, whose corresponding noise $\x^{\dagger}_T$ is then used to generate the desired motion. Note that we freeze the feature encoder and diffusion model.}
    \label{fig:retoverview}
\end{figure}

At retrieval time, we compute the cosine similarity between each dictionary feature and the text embedding:
\begin{equation}
s(\z, c) = \frac{\z  \cdot \mathcal{F}_\text{Text}(c)}{\lVert \z \rVert \lVert \mathcal{F}_\text{Text}(c) \rVert}\,,
\end{equation}
where $\mathcal{F}_\text{Text}$ is the text encoder of the retrieval model. We rank all entries by similarity and define the top-$k$ candidate set:
\begin{equation}
\mathcal{N}_k(c) = \operatorname*{Top\text{-}k}_{(\z^i_l, \x^i_T) \in \mathcal{D}}\, s(\z^i_l, c)\,.
\label{eq:retrieval}
\end{equation}
where $\mathcal{D}\in\{\mathcal{D}_{\mathrm{reg}},\mathcal{D}_{\mathrm{adj}}\}$. The retrieved feature $\z^{\dagger}$ is then sampled uniformly from $\mathcal{N}_k(c)$. When $k{=}1$, this reduces to deterministic argmax retrieval.
Along with the provided text prompt $c$, the initial noise $\x^{\dagger}_T$ corresponding to $\z^{\dagger}$ is used to generate the final motion (Fig.~\ref{fig:retoverview}).
 We denote retrieval using $\mathcal{D}_{reg}$ as ``Retrieval'' and using $\mathcal{D}_{adj}$ as ``Frame-adjusted retrieval (FAR)''.

\paragraph{Noise perturbation.}
To avoid over-concentrating on a small subset of seeds while preserving semantics, we perturb the retrieved noise with small Gaussian jitter:
\begin{equation}
\x'_T = \x^{\dagger}_T + \boldsymbol{\eta}, \quad \boldsymbol{\eta} \sim \mathcal{N}(\mathbf{0}, \sigma^2 \mathbf{I})\,,
\label{eq:perturbation}
\end{equation}
where $\sigma$ controls the stochasticity.

\subsection{Stage 2: Winning Noise Ticket Refinement}
\label{sec:stage_2}
We refine the retrieved noise $\x^{\dagger}_T$ by maximizing text--motion similarity (Fig.~\ref{fig:optoverview}):
\begin{equation}
\mathcal{L}_{sim} = 1 - s(\mathcal{F}_\text{Motion}(\x_0), c) \,.
\end{equation}
As in retrieval, $\x_0$ can be truncated to the target length before encoding.
Directly optimizing this similarity loss can move the noise outside the Gaussian typical set (e.g., by drifting its empirical mean/variance), where the diffusion model is poorly calibrated and motion quality/diversity may degrade. We therefore add a KL-divergence term that penalizes deviations from the standard normal:
\begin{equation}
\mathcal{L}_{KL} = D_{KL}\!\left(\mathcal{N}(\mu(\x_T), \sigma^2(\x_T)) \| \mathcal{N}(0, 1)\right)\,,
\end{equation}
where $\mu(\x_T)$ and $\sigma^2(\x_T)$ are the empirical mean and variance of the noise tensor. The full objective is:
\begin{equation}
\mathcal{L} = \mathcal{L}_{sim} + \lambda \mathcal{L}_{KL}\,,
\label{eq:optimization}
\end{equation}
where $\lambda$ is the balancing parameter. We freeze the motion diffusion model and optimize only the initial noise to obtain $\x^*_T$ aligned with the user-provided text.

\begin{figure}[t]
    \centering
    \includegraphics[width=1\linewidth]{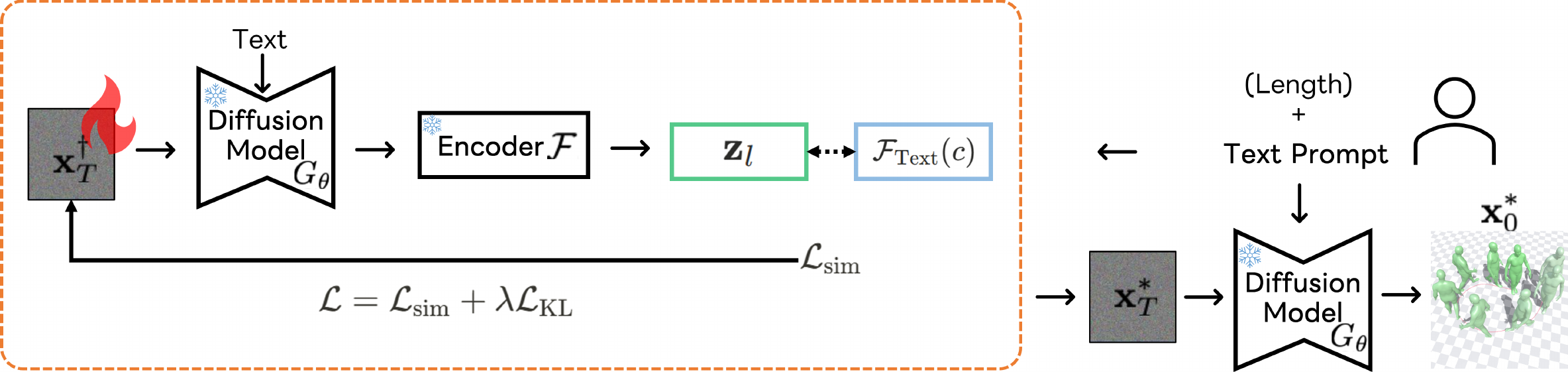}
    \caption{Overview of Winning Noise Ticket Refinement. The retrieved noise $\x^{\dagger}_T$ can further be refined to improve the quality of the output using input text $c$, and optionally the desired motion length $l$. The loss $\mathcal{L}$ is minimized to obtain the refined noise $\x^{*}_T$. }
    \label{fig:optoverview}
\end{figure}

\subsection{LoRA-Based Noise Refiner}
Iterative refinement runs diffusion sampling at every optimization step, which is costly (see AITS in Table~\ref{tab:win_far_nulltext}). As an amortized alternative, we train a Low-Rank Adaptation (LoRA)~\cite{hu2022lora} module that predicts a noise correction in a single forward pass, following~\cite{eyring2025noise}. Unlike the retrieval and optimization stages, the LoRA refiner requires an additional training step and is therefore not training-free, but it enables real-time inference:
\begin{equation}
    \hat{\x}_T = \x_T + R_\phi(\x_T)\,.
\end{equation}
Here, $R_\phi$ predicts a residual correction so that $\hat{\x}_T$ approximates the refined noise produced by retrieval and iterative optimization.
Only the LoRA weights $\phi$ are updated while all diffusion parameters $\theta$ remain frozen.
During training, $\x_T$ is sampled from the standard Gaussian prior rather than from the retrieval dictionary, so no explicit noise-to-noise supervision pairs are required.
The refiner minimizes a loss analogous to Eq.~\ref{eq:optimization}:
\begin{equation}
\begin{split}
\mathcal{L}_{LoRA} &=
1 - \cos\!\left(
\mathcal{F}_\text{Motion}(G_\theta(\hat{\x}_T, c)),
\mathcal{F}_\text{Text}(c)
\right) + \lambda \lVert R_\phi(\x_T) \rVert^2_2,
\end{split}
\end{equation}
where the second term regularizes the correction magnitude, helping the corrected noise maintain proximity to the Gaussian prior.
At inference, the refiner can be applied either to random noise or to the retrieved $\x^{\dagger}_T$ from Stage~1.

\section{Experiments}

In this section, we first verify the intrinsic semantics of winning noise tickets through null-prompt generation (Sec.~5.1), then evaluate the full WINRO pipeline on text-conditioned benchmarks (Sec.~5.2). We further examine compositional and long-duration generation (Sec.~5.3) and demonstrate task adaptability to motion stylization and controllable generation (Sec.~5.4). Experimental details are provided in the supplementary material (Sec.~B).

\subsection{Intrinsic Semantics of Winning Noise Tickets}
\paragraph{Generation with Null Prompt.}
To isolate the effect of noise selection from text conditioning, we first retrieve winning noise tickets using text embeddings but generate motions under a \textit{null prompt}, so that any observed semantic alignment must originate from the noise itself.
Following~\Sref{sec:stage_1}, we construct a noise dictionary containing 10{,}000 randomly sampled noise–motion pairs, where each feature is computed by encoding a null-prompt motion.
For efficiency, truncated samples are created every 4 frames, resulting in motions of lengths $40, 44, \ldots, 196$ for Frame-Adjusted Retrieval (FAR).
All noises are drawn from a standard Gaussian distribution, and no ground-truth motion data is used.

We conduct experiments on the HumanML3D dataset~\cite{Guo_2022_CVPR_humanml3d}.
Two representative diffusion backbones are adopted: MDM-50step~\cite{tevet2023human} and MotionLCM~\cite{motionlcm}.
The pretrained TMR model~\cite{petrovich2023tmr} is used to embed both motions and texts into a shared 256-dimensional feature space for retrieval. Unless otherwise noted, we use top-1 retrieval ($k{=}1$). Details are provided in the supplementary material.

Following prior works~\cite{tevet2023human,chen2023executing}, we report FID, R-Precision, Diversity, Multimodal Distance (MMDist), MultiModality, and additionally include AITS. FID measures the distributional distance between generated and real motions in a learned feature space, reflecting overall motion quality. R-Precision and MMDist evaluate text--motion alignment, while Diversity measures variation across generated motions. MultiModality measures variation among motions generated from the same text; see Section~C. \textbf{AITS (Average Inference Time per Sentence)} reports batch-size-one wall-clock inference time in seconds, excluding data/model loading and the one-time noise dictionary construction. All reported experiments use a single NVIDIA H100 GPU, except where explicitly stated.

\begin{table}[t]
    \centering
    \caption{Performance of MDM-50step and MotionLCM using only WIN retrieval on HumanML3D under generation with \textit{null} prompts. Results are averaged over 20 evaluation runs with 95\% confidence intervals.}
    \setlength{\tabcolsep}{1pt}
    \resizebox{\linewidth}{!}{
    \begin{tabular}{lcccccc}
        \toprule
        \textbf{Method} &  \textbf{FID} $\downarrow$ & \textbf{R-Top1} $\uparrow$ & \textbf{R-Top2} $\uparrow$ & \textbf{R-Top3} $\uparrow$ & \textbf{Diversity} $\rightarrow$ & \textbf{MMDist} $\downarrow$ \\
        \midrule \midrule
        GT & $0.002^{\scriptstyle\pm.000}$ & $0.511^{\scriptstyle\pm.002}$ & $0.704^{\scriptstyle\pm.003}$ & $0.798^{\scriptstyle\pm.002}$ & $9.458^{\scriptstyle\pm.088}$ & $2.970^{\scriptstyle\pm.007}$ \\
        \midrule
        \rowcolor[gray]{0.90} MDM-50step  & $4.956^{\scriptstyle\pm.150}$ & $0.037^{\scriptstyle\pm.002}$ & $0.070^{\scriptstyle\pm.002}$ & $0.102^{\scriptstyle\pm.004}$ & $8.101^{\scriptstyle\pm.087}$ & $8.932^{\scriptstyle\pm.039}$ \\
        w/ WIN-R & $4.500^{\scriptstyle\pm.228}$ & $0.250^{\scriptstyle\pm.006}$ & $0.385^{\scriptstyle\pm.009}$ & $0.474^{\scriptstyle\pm.008}$ & $8.482^{\scriptstyle\pm.068}$ & $5.540^{\scriptstyle\pm.057}$ \\
        w/ WIN-FAR & $\mathbf{4.458}^{\scriptstyle\pm.200}$ & $\mathbf{0.274}^{\scriptstyle\pm.006}$ & $\mathbf{0.415}^{\scriptstyle\pm.006}$ & $\mathbf{0.507}^{\scriptstyle\pm.007}$ & $\mathbf{8.547}^{\scriptstyle\pm.059}$ & $\mathbf{5.325}^{\scriptstyle\pm.046}$ \\
        \midrule
        \rowcolor[gray]{0.90} MotionLCM & $8.460^{\scriptstyle\pm.072}$ & $0.034^{\scriptstyle\pm.001}$ & $0.067^{\scriptstyle\pm.002}$ & $0.099^{\scriptstyle\pm.002}$ & $6.646^{\scriptstyle\pm.064}$ & $8.313^{\scriptstyle\pm.010}$ \\
        w/ WIN-R  & $1.681^{\scriptstyle\pm.071}$ & $0.404^{\scriptstyle\pm.002}$ & $0.597^{\scriptstyle\pm.004}$ & $0.710^{\scriptstyle\pm.003}$ & $8.386^{\scriptstyle\pm.029}$ & $3.613^{\scriptstyle\pm.020}$ \\
        w/ WIN-FAR  & $\mathbf{1.569}^{\scriptstyle\pm.071}$ & $\mathbf{0.424}^{\scriptstyle\pm.003}$ & $\mathbf{0.617}^{\scriptstyle\pm.004}$ & $\mathbf{0.728}^{\scriptstyle\pm.004}$ & $\mathbf{8.424}^{\scriptstyle\pm.039}$ & $\mathbf{3.485}^{\scriptstyle\pm.021}$ \\
        \bottomrule
    \end{tabular}
    }
    \label{tab:retrieval_null}
\end{table}

As shown in Table~\ref{tab:retrieval_null}, retrieved winning noise tickets enable both MDM-50step and MotionLCM to produce semantically aligned motions even under null prompts, where baselines yield near-zero alignment. This validates the Winning Noise Ticket Hypothesis (Sec.~\ref{sec:winning_noise}) by showing that semantic information can be partially encoded in the initial noise itself, independent of text conditioning.

\subsection{Generation with Refined Winning Noise Tickets}
We next examine whether this effect carries over to text-conditioned generation with the full WINRO pipeline.
The retrieved noise $\x_T^{\dagger}$ from the FAR dictionary is optimized using the KL-regularized objective (Eq.~\ref{eq:optimization}), yielding a refined noise $\x_T^{*}$.
The optimization is performed at inference time without modifying model parameters. Details are in the supplementary material.
Alternatively, a \textit{LoRA Noise Refiner} can approximate this refinement in a single feed-forward pass; when combined with retrieval (WIN-RLoRA), it amortizes the iterative optimization while preserving the benefit of a semantically aligned initialization.

\begin{table}[t]
    \centering
    \caption{Performance of MDM-50step and MotionLCM with WINRO on HumanML3D under text-conditioned generation. MotionLCM uses 1-step inference. Results are averaged over 20 evaluation runs with 95\% confidence intervals.}
    \setlength{\tabcolsep}{1pt}
    \resizebox{\linewidth}{!}{
    \begin{tabular}{lccccccc}
        \toprule
        \textbf{Method} & \textbf{FID} $\downarrow$ & \textbf{R-Top1} $\uparrow$ & \textbf{R-Top2} $\uparrow$ & \textbf{R-Top3} $\uparrow$ & \textbf{Div.} $\rightarrow$ & \textbf{MMDist} $\downarrow$ & \textbf{AITS} $\downarrow$\\
        \midrule \midrule
        GT & $0.002^{\scriptstyle\pm.000}$ & $0.511^{\scriptstyle\pm.002}$ & $0.704^{\scriptstyle\pm.003}$ & $0.798^{\scriptstyle\pm.002}$ & $9.458^{\scriptstyle\pm.088}$ & $2.970^{\scriptstyle\pm.007}$  & -\\
        \midrule
        \rowcolor[gray]{0.90} MDM-50step  & $0.438^{\scriptstyle\pm.009}$ & $0.469^{\scriptstyle\pm.002}$ & $0.661^{\scriptstyle\pm.002}$ & $0.763^{\scriptstyle\pm.002}$ & $10.025^{\scriptstyle\pm.066}$ & $3.258^{\scriptstyle\pm.009}$  & \textbf{0.265} \\
        w/ WIN-R & $0.250^{\scriptstyle\pm.007}$ & $0.505^{\scriptstyle\pm.002}$ & $0.702^{\scriptstyle\pm.003}$ & $0.799^{\scriptstyle\pm.003}$ & $10.303^{\scriptstyle\pm.048}$ & $3.047^{\scriptstyle\pm.012}$ & -- \\
        w/ WIN-FAR & $0.248^{\scriptstyle\pm.009}$ & $0.512^{\scriptstyle\pm.003}$ & $0.706^{\scriptstyle\pm.003}$ & $0.804^{\scriptstyle\pm.002}$ & $10.295^{\scriptstyle\pm.033}$ & $3.013^{\scriptstyle\pm.011}$  & 0.268 \\
        w/ WINRO & $0.115^{\scriptstyle\pm.002}$ & $0.539^{\scriptstyle\pm.002}$ & $0.727^{\scriptstyle\pm.004}$ & $0.817^{\scriptstyle\pm.002}$ & $9.657^{\scriptstyle\pm.026}$ & $2.883^{\scriptstyle\pm.010}$  & 81.408 \\
        w/ WIN-RLoRA & $\mathbf{0.102}^{\scriptstyle\pm.002}$ & $\mathbf{0.546}^{\scriptstyle\pm.003}$ & $\mathbf{0.741}^{\scriptstyle\pm.003}$ & $\mathbf{0.832}^{\scriptstyle\pm.002}$
        & $\mathbf{9.577}^{\scriptstyle\pm.025}$ & $\mathbf{2.846}^{\scriptstyle\pm.007}$ & 0.280 \\ \midrule
        \rowcolor[gray]{0.90} MotionLCM & $0.093^{\scriptstyle\pm.004}$ & $0.548^{\scriptstyle\pm.003}$ & $0.741^{\scriptstyle\pm.003}$ & $0.835^{\scriptstyle\pm.002}$ & $9.566^{\scriptstyle\pm.069}$ & $2.763^{\scriptstyle\pm.008}$ & \textbf{0.029} \\
        w/ WIN-R & $0.091^{\scriptstyle\pm.004}$ & $0.566^{\scriptstyle\pm.003}$ & $0.762^{\scriptstyle\pm.002}$ & $0.852^{\scriptstyle\pm.002}$ & $9.424^{\scriptstyle\pm.028}$ & $2.699^{\scriptstyle\pm.007}$ & -- \\
        w/ WIN-FAR  & $0.083^{\scriptstyle\pm.003}$ & $0.572^{\scriptstyle\pm.003}$ & $0.767^{\scriptstyle\pm.002}$ & $0.855^{\scriptstyle\pm.001}$ & $\mathbf{9.474}^{\scriptstyle\pm.028}$ & $2.674^{\scriptstyle\pm.006}$  & 0.033 \\
        w/ WINRO & $\mathbf{0.072}^{\scriptstyle\pm.002}$ & $\mathbf{0.580}^{\scriptstyle\pm.002}$ & $\mathbf{0.775}^{\scriptstyle\pm.002}$ & $\mathbf{0.861}^{\scriptstyle\pm.001}$ & $9.242^{\scriptstyle\pm.026}$ & $\mathbf{2.609}^{\scriptstyle\pm.005}$ & 6.870 \\
        w/ WIN-RLoRA & $0.083^{\scriptstyle\pm.002}$ & $\mathbf{0.580}^{\scriptstyle\pm.002}$ & $0.772^{\scriptstyle\pm.001}$ & $0.860^{\scriptstyle\pm.001}$
        & $9.550^{\scriptstyle\pm.032}$ & $2.649^{\scriptstyle\pm.005}$ &  0.043 \\
        \bottomrule
    \end{tabular}
    }
    \label{tab:win_far_nulltext}

\end{table}

Table~\ref{tab:win_far_nulltext} summarizes the text-conditioned generation results. Each stage of WINRO yields consistent improvements across both diffusion backbones.
Plain retrieval (WIN-R) already improves text alignment over the baselines, reducing the FID for MDM-50step from $0.438$ to $0.250$. Frame-adjusted retrieval with perturbation (WIN-FAR) further improves upon WIN-R (FID $0.248$, R-Top1 $0.512$); the perturbation (Eq.~\ref{eq:perturbation}) restores local diversity around the retrieved noise, mitigating the distribution shift caused by discrete selection.
Subsequent refinement (WINRO) further strengthens alignment: for MDM-50step, the FID decreases to $0.115$ and R-Top1 rises to $0.539$; for MotionLCM, WINRO achieves the best R-Top1 of $0.580$ and MMDist of $2.609$. Since the optimization objective already incorporates KL regularization (Eq.~\ref{eq:optimization}), we omit additional perturbation during refinement.
LoRA-based refinement with retrieval and perturbation (WIN-RLoRA) provides an amortized alternative to full WINRO refinement, which is iterative and suited for offline or quality-critical use. Although WIN-RLoRA requires an additional training step and is therefore not training-free, it avoids iterative test-time optimization. It further improves over WINRO on MDM-50step and remains competitive on MotionLCM, while reducing AITS from $81.408$ to $0.280$ and from $6.870$ to $0.043$, respectively.

Note that the retrieval model (TMR~\cite{petrovich2023tmr}) used for dictionary construction and refinement is independent of the evaluation model used for the HumanML3D metrics. We further verify that improvements are not specific to the choice of retrieval backbone by repeating the experiment with an alternative encoder; details are provided in the supplementary material.
The noise dictionary is constructed once offline per backbone: approximately 4 minutes for MDM-50step and 2 minutes for MotionLCM on a single A100 GPU, after which per-query inference requires only retrieval (negligible) and optional refinement.
Additional analyses and ablations are provided in the supplementary material.

\begin{figure}[t]
    \centering
    \includegraphics[width=\linewidth]{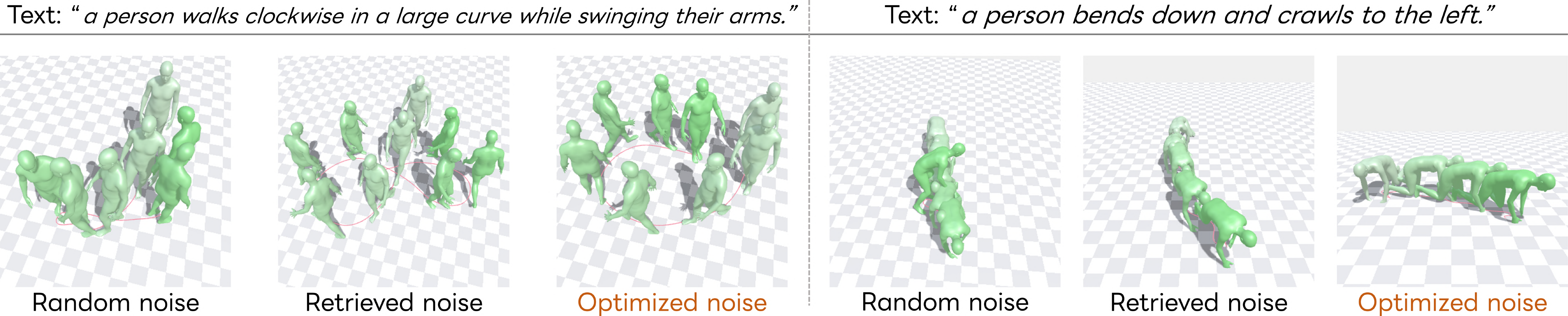}
    \caption{Comparison of motions generated with random, retrieved, and refined initial noise for each text prompt, using MDM-50step with and without our method. The texts shown above each row are used for generation. We recommend viewing the supplementary video for better visualization of motion dynamics.}
    \label{fig:basecomparison}
\end{figure}

\paragraph{Qualitative Results.}
Fig.~\ref{fig:basecomparison} compares motions generated with random noise, retrieved noise (WIN-FAR), and refined noise (WINRO) using MDM-50step.

In the first example, the prompt demands a clockwise walk. Random noise produces a generic walking motion with no clear trajectory, while the retrieved noise captures a partial arc but eventually turns in the wrong direction. In contrast, WINRO generates a smooth clockwise arc that closely matches the intended motion.
In the second example, the prompt requires the person to bend down and crawl to the left. Random noise yields a crawling motion that ends in an unintended standing pose, and the retrieved noise produces only slight leftward motion. With WINRO refinement, the model is able to generate a consistent leftward crawl as instructed.

These qualitative results illustrate the complementary roles of retrieval and refinement: retrieval selects a noise seed with an inherent bias toward the desired behavior, while refinement further adjusts it to satisfy the prompt more faithfully. Additional examples
are in the supplementary material.

\subsection{Compositional and Long-Duration Generation.}
Practical motion generation often requires composing multiple actions within a single sequence: actions may be arranged in temporal order (\eg, ``walk forward, then sit down'') or performed simultaneously on different body parts (\eg, ``wave the right hand while walking'').
The multi-track timeline (MTT) formulation~\cite{petrovich2024stmc} unifies these settings by assigning each text prompt to a temporal interval on parallel tracks, allowing overlapping and sequential actions to be specified in a single structured input. Generating faithful motions from such timelines is challenging because the model must maintain semantic alignment across all intervals, especially as the total duration grows.
WINRO naturally fits this setting, since it can also be applied to multiple interval-level text constraints.

\paragraph{Method.}
We evaluate on the MTT benchmark~\cite{petrovich2024stmc} under two settings: (1) single-text compositional generation ($\le$196 frames) using MDM-50step, where the full WINRO pipeline is applied, and (2) multi-track long-duration generation using STMC~\cite{petrovich2024stmc}, where only the optimization stage is used because constructing a noise dictionary from variable-length null-prompt outputs is computationally impractical at extended durations. For the multi-track setting, the similarity loss is defined as the average text--motion similarity over timeline intervals.
Each motion subsequence is encoded and compared with its corresponding text prompt, enabling the optimization to maintain semantic consistency across segments.

\paragraph{Results.}
As shown in Table~\ref{tab:mtt}, both configurations yield consistent improvements, with particularly large gains in long-duration generation (R@1: $33.6 \to 62.1$). These results suggest that winning noise tickets carry semantic priors that remain effective even in extended and compositional settings, where semantic drift is otherwise difficult to control. Fig.~\ref{fig:stmc_qualitative} shows qualitative examples: STMC alone suffers from semantic drift in later segments, whereas WIN-O maintains faithful alignment with each text prompt throughout the timeline.

\begin{table}[t]
\centering
\caption{Results on the MTT benchmark~\cite{petrovich2024stmc}.}
\setlength{\tabcolsep}{3pt}
\resizebox{\linewidth}{!}{
\begin{tabular}{lcc|cccc|cc}
\toprule
& \multicolumn{2}{c|}{\textbf{Input type}} & \multicolumn{4}{c|}{\textbf{Per-crop semantic correctness}} & \multicolumn{2}{c}{\textbf{Realism}} \\
\textbf{Method} & \#tracks & \#crops & R@1 $\uparrow$ & R@3 $\uparrow$ & M2T $\uparrow$ & M2M $\uparrow$ & FID $\downarrow$ & Trans. $\downarrow$ \\
\midrule \midrule
Ground Truth & -- & -- & 55.0 & 73.3 & 0.748 & 1.000 & 0.000 & 1.5 \\
\midrule
\rowcolor[gray]{0.90} MDM-50step & Single & Single & 12.2 & 24.6 & 0.571 & 0.561 & 0.543 & 2.2 \\
w/ WINRO & Single & Single & \textbf{20.7} & \textbf{38.1} & \textbf{0.661} & \textbf{0.597} & \textbf{0.492} & \textbf{2.1} \\
\midrule
\rowcolor[gray]{0.90} w/ STMC~\cite{petrovich2024stmc} & Multi & Multi & 33.6 & 55.3 & 0.687 & 0.656 & 0.480 & 1.7 \\
w/ STMC + WIN-O & '' & '' & \textbf{62.1} & \textbf{79.1} & \textbf{0.807} & \textbf{0.702} & \textbf{0.452} & \textbf{1.4} \\
\bottomrule
\end{tabular}
}
\label{tab:mtt}
\end{table}

\begin{figure}[t]
    \centering
    \includegraphics[width=\linewidth]{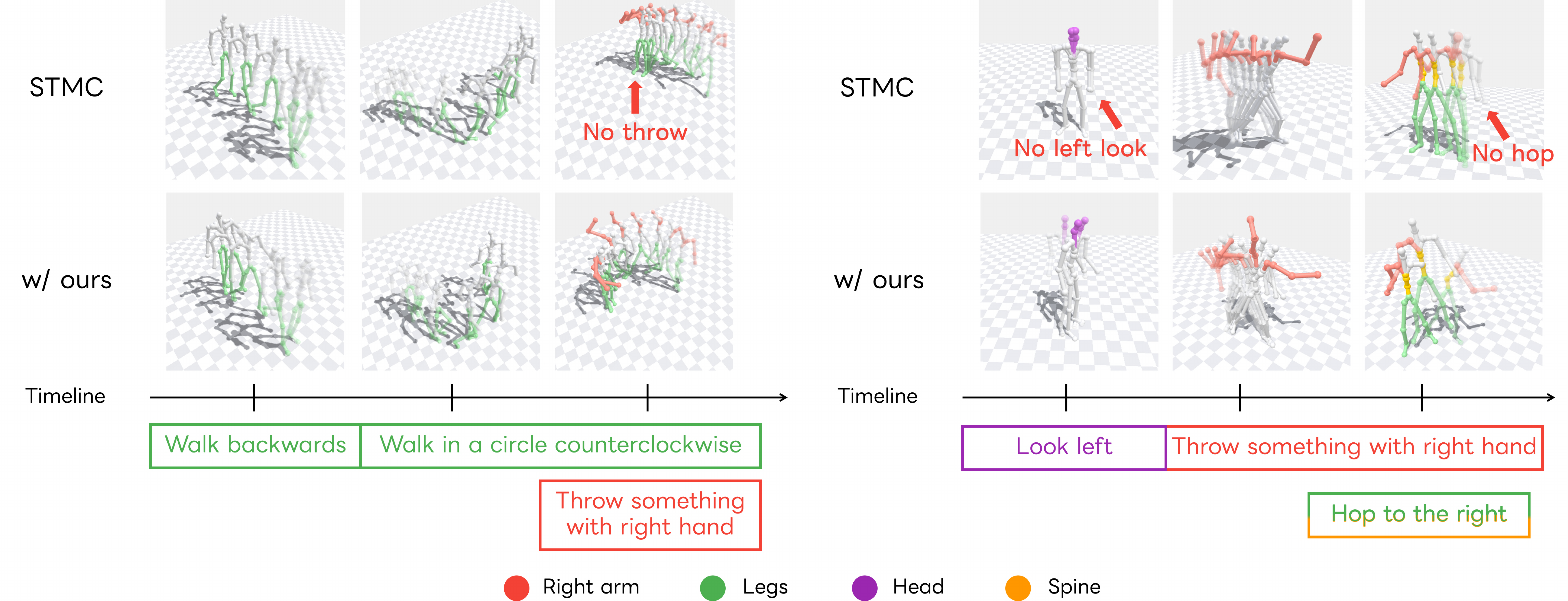}
    \caption{Comparison of motions generated by STMC~\cite{petrovich2024stmc} with and without initial noise refinement on multi-track long-duration generation. STMC produces motions that are not faithful to the text prompt of each segment (\eg, missing throw or hop actions), while refining the initial noise improves alignment.}
    \label{fig:stmc_qualitative}
\end{figure}

\subsection{Task Adaptability}
\label{sec:task_adapt}
As noted in Sec.~\ref{sec:method}, the retrieval and optimization stages are modular: the dictionary can be replaced with a domain-specific one or omitted,
and the optimization objective can incorporate additional loss terms. We demonstrate this flexibility on two tasks.

\subsubsection{Motion Stylization.}
Motion stylization aims to generate motions that follow a given text prompt while exhibiting a particular movement style (\eg, performing a ``walking'' motion in a ``proud'' or ``depressed'' manner), guided by a reference style motion.

\begin{figure}[t]
    \centering
    \includegraphics[width=\linewidth]{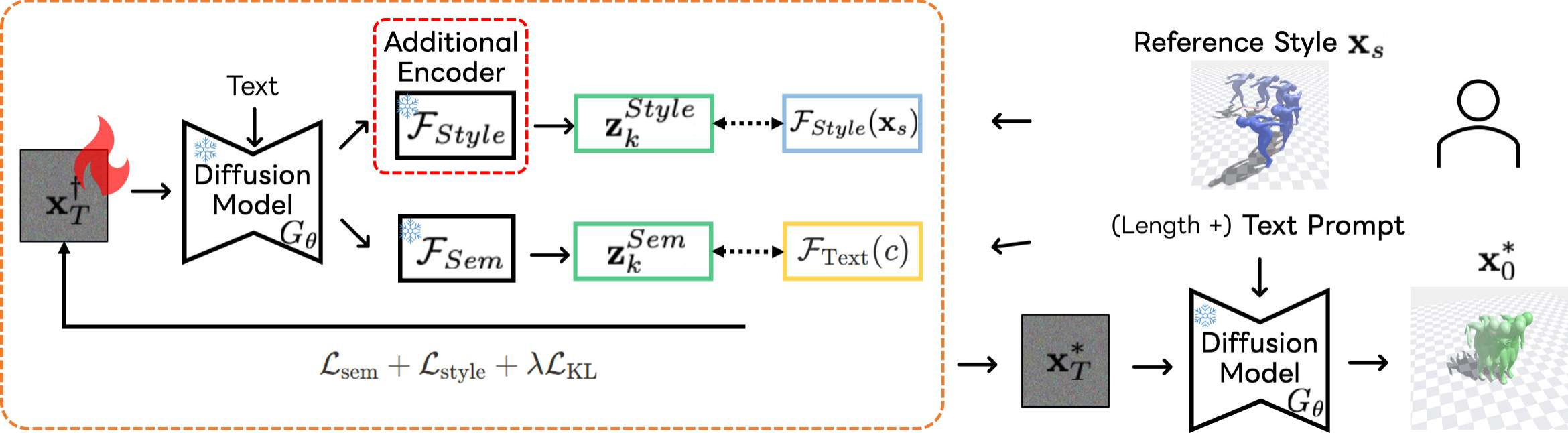}
    \caption{Overview of WINRO for motion stylization. The generic noise dictionary is replaced with a style-specific one built from 100STYLE motions~\cite{mason2022local} via DDIM inversion. The optimization jointly minimizes semantic loss $\mathcal{L}_\text{sim}$, style loss $\mathcal{L}_\text{style}$, and KL regularization term $\lambda\mathcal{L}_\text{KL}$.}
    \label{fig:styleoverview}
\end{figure}

\paragraph{Method.}
We extend WINRO by constructing a style-specific noise dictionary $\mathcal{D}_s$: we apply DDIM inversion to motions from the 100STYLE dataset~\cite{mason2022local}, so that each entry inherently encodes stylistic attributes.
Given a text prompt, we retrieve the best-matching entry from $\mathcal{D}_s$ using the same text-based retrieval as in Eq.~\ref{eq:retrieval}, then refine the retrieved noise by jointly optimizing the semantic loss $\mathcal{L}_{sim}$ (Eq.~\ref{eq:optimization}) and a style loss $\mathcal{L}_{style} = 1 - \cos(\mathcal{F}_\text{Style}(\x^{\dagger}_0), \mathcal{F}_\text{Style}(\mathbf{x}_s))$, where $\mathbf{x}_s$ is the reference style motion.
We compare WIN-O (optimization only from random noise) and WINRO (retrieval from $\mathcal{D}_s$ + optimization). Fig.~\ref{fig:styleoverview} illustrates the pipeline.
Since WINRO is model-agnostic, it can be applied to any diffusion-based stylization backbone. We adopt SMooDi~\cite{zhong2024smoodi} for its public availability.

\begin{table}[t]
\centering
\setlength{\tabcolsep}{1pt}
\caption{Effect of WINRO on motion stylization using SMooDi as the backbone.}
\begin{tabular}{lcccccc}
\toprule
\textbf{Method} & \textbf{FID} $\downarrow$ & \textbf{Foot Skating} $\downarrow$ & \textbf{MMDist} $\downarrow$ & \textbf{R-Top3} $\uparrow$ & \textbf{Diversity} $\rightarrow$ & \textbf{SRA} $\uparrow$ \\
\midrule
\rowcolor[gray]{0.90} SMooDi~\cite{zhong2024smoodi}  & 1.609 & 0.124 & 4.477 & 0.571 & 9.235  & 72.418 \\
w/ WIN-O & \textbf{0.551}  & \textbf{0.104} & \textbf{3.550} & \textbf{0.717} & \textbf{9.071}  & 68.029 \\
w/ WINRO & 1.066  & 0.129 & 4.055 & 0.633 & 7.766  & \textbf{75.818} \\
\bottomrule

\end{tabular}
\label{tab:stylization}
\end{table}

\begin{figure}[t]
    \centering
    \includegraphics[width=1\linewidth]{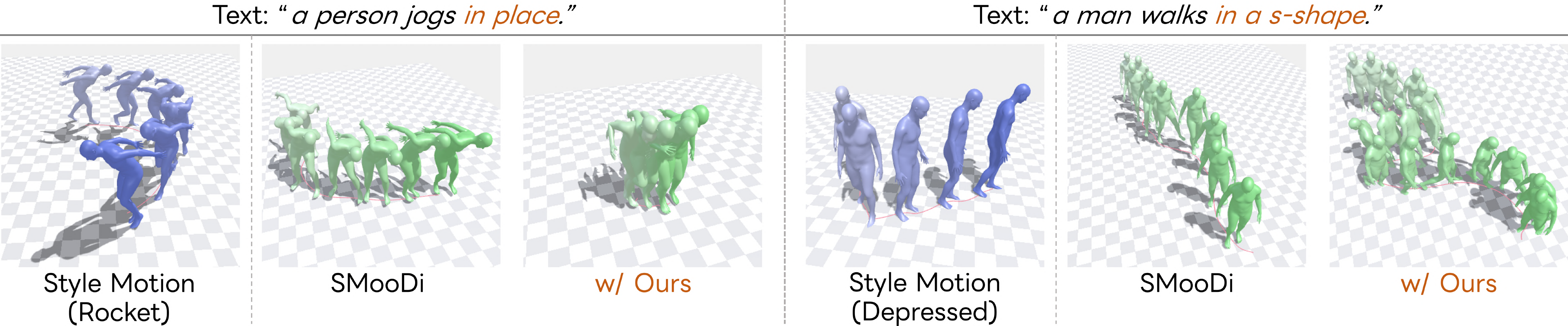}
    \caption{Qualitative results from the stylization experiment.
    Reference style motion at the left and the text prompt at the top of each example are used for generation.}
    \label{fig:stylecomparison}
\end{figure}

\paragraph{Results.}
Table~\ref{tab:stylization} shows that both variants improve upon the SMooDi backbone without modifying its parameters.
WIN-O, starting from random noise, optimizes freely toward text alignment and achieves the best FID ($0.551$ vs.\ $1.609$), MMDist, and R@3. WINRO, starting from the style-specific dictionary, begins closer to the target style; accordingly, it attains the highest Style Recognition Accuracy (SRA) ($75.82$ vs.\ $72.42$), confirming stronger style adherence, while text-alignment metrics remain competitive.
The two variants thus offer complementary trade-offs: WIN-O prioritizes text fidelity, while WINRO prioritizes style preservation---both operating solely at the noise level.

\paragraph{Qualitative Results}
We qualitatively compare stylized motions in Fig.~\ref{fig:stylecomparison}.
In the first example, both SMooDi and our method reproduce the ``Rocket'' style, but SMoodi's output drifts semantically, producing a wandering motion instead.
Our framework selects a more suitable initial noise, yielding a coherent ``jogging in place'' motion that respects both style and text.
A similar trend is observed in the second example, where the baseline captures style but loses semantic context, while our method maintains alignment between content and style.

\subsubsection{Controllable Motion Generation.}
Beyond text semantics, practical applications often require motions to satisfy explicit spatial constraints, such as reaching a specific location or avoiding obstacles.
The optimization objective can incorporate such geometric or physical constraints from prior works~\cite{karunratanakul2024dno, liu2024programmable}. Concretely, we augment Eq.~\ref{eq:optimization} with a constraint term $\beta\,\mathcal{L}_{\mathrm{ctrl}}$ (\eg, the waypoint/target-reaching objective of ProgMoGen~\cite{liu2024programmable}). As shown in Fig.~\ref{fig:control}, this combined objective finds an initial noise that satisfies both textual semantics and external spatial constraints, such as reaching target locations.

\begin{figure}[t]
    \centering
    \includegraphics[width=\linewidth]{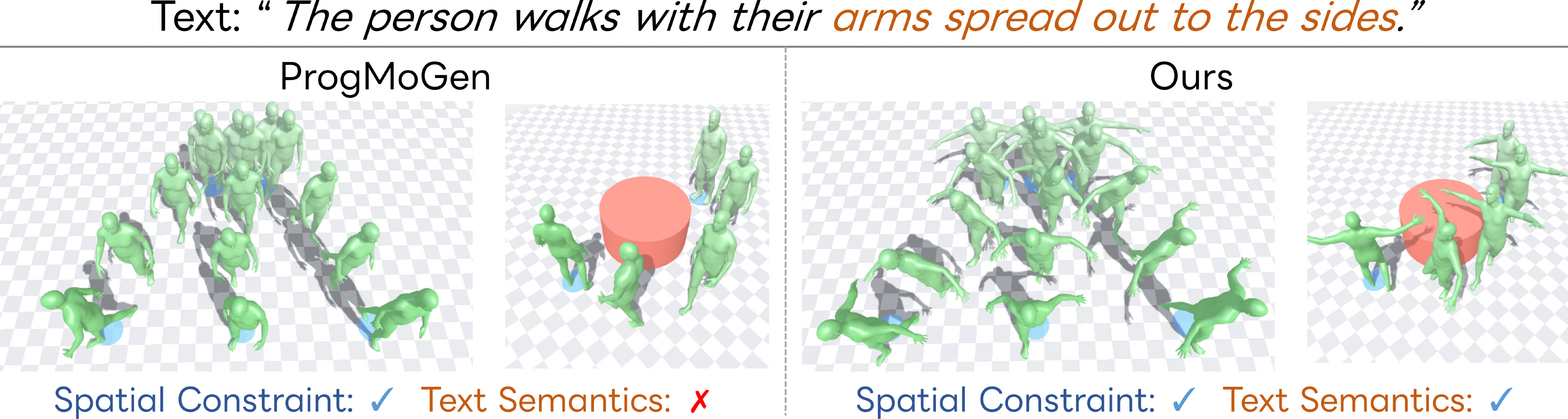}
    \caption{Comparison of motions generated using the initial noise optimized by ProgMoGen~\cite{liu2024programmable} and the refined noise produced by the proposed WINRO method. WINRO produces motions that better satisfy both textual and spatial constraints.}
    \label{fig:control}
\end{figure}

\section{Limitations}
Since WINRO does not modify the base model, it cannot address systematic failure modes of the diffusion backbone and is limited to guiding generation within its learned manifold.
The iterative refinement stage incurs substantial inference overhead, as each optimization step requires a full diffusion generation pass; the LoRA-based refiner mitigates this for latency-sensitive applications, but requires an additional training step.
Finally, the current motion-text retrieval model treats motions holistically, making it difficult to pinpoint fine-grained temporal properties in noise patterns. Investigating retrieval models with temporal localization is a promising direction for future work.

\section{Conclusion}
We present WINRO, a training-free and model-agnostic framework that exploits the semantic structure of the initial noise space to improve diffusion-based text-to-motion generation. By retrieving and refining winning noise tickets, WINRO consistently improves text--motion alignment across MDM and MotionLCM without retraining the base model. WINRO also improves temporal robustness on the MTT benchmark. Its modular design generalizes to motion stylization and spatial constraint satisfaction, demonstrating that principled noise selection is a versatile and complementary tool for controllable motion synthesis.

\section*{Acknowledgements}
This study was carried out using the TSUBAME4.0 supercomputer at Institute of Science Tokyo.

\clearpage
\bibliographystyle{splncs04}
\bibliography{main}

\clearpage
\renewcommand{\thesection}{\Alph{section}}
\renewcommand{\theHsection}{supp.\Alph{section}}
\renewcommand{\theHsubsection}{supp.\Alph{section}.\arabic{subsection}}
\renewcommand{\theHsubsubsection}{supp.\Alph{section}.\arabic{subsection}.\arabic{subsubsection}}
\setcounter{section}{0}

\begin{center}
    \Large\textbf{Supplementary Material}\\[6pt]
    \large\textbf{Retrieving and Refining Winning Noise Tickets for Diffusion-Based Motion Generation}
\end{center}
\vspace{12pt}

\section{Theoretical Perspective on Winning Noise Tickets}
\label{sec:ticket_theory}
We address the question \emph{why certain noises encode motion semantics} through the deterministic (DDIM/ODE) view of diffusion sampling.
The key observation is that deterministic sampling defines a \emph{flow map} from initial noise to a motion trajectory; under mild smoothness assumptions this map is stable, so high semantic alignment cannot occur at isolated points but instead occupies neighborhoods (\ie, semantic \emph{regions}) in the noise space. For motion, where denoising couples all frames, such neighborhoods naturally correspond to coherent action-level trajectories rather than per-frame artifacts. This, in turn, justifies (i) dictionary search over many seeds and (ii) KL-regularized refinement to preserve the Gaussian prior.

\paragraph{Setup.}
For a deterministic sampler, let $\x_0 = G_\theta(\x_T,c)$ denote the generated motion from initial noise $\x_T\in\mathbb{R}^{d}$ ($d{=}FD$) and text condition $c$.
Let $\mathcal{F}_\text{Motion}$ and $\mathcal{F}_\text{Text}$ be a motion--text retrieval model's encoders, and define the normalized embeddings
\begin{equation}
\Phi_c(\x_T)=\frac{\mathcal{F}_\text{Motion}(G_\theta(\x_T,c))}{\lVert \mathcal{F}_\text{Motion}(G_\theta(\x_T,c))\rVert},\quad
\tau(c)=\frac{\mathcal{F}_\text{Text}(c)}{\lVert \mathcal{F}_\text{Text}(c)\rVert}\,.
\end{equation}
We consider the cosine alignment score
\begin{equation}
S_c(\x_T)=\Phi_c(\x_T)^\top \tau(c)\in[-1,1]\,.
\end{equation}
In particular, our retrieval dictionary (Sec.~\ref{sec:stage_1}) uses the \emph{null} score $\tilde{S}_c(\x_T)=\Phi_{\varnothing}(\x_T)^\top\tau(c)$, which sets only the diffusion condition to $\varnothing$ while retaining the text query $c$.

\begin{definition}[Winning region]
For a prompt $c$ and threshold $\gamma$, define the super-level set
$\mathcal{W}_{c,\gamma}=\{\x_T:\tilde{S}_c(\x_T)\ge\gamma\}$.
We call $\x_T\in\mathcal{W}_{c,\gamma}$ a \emph{winning noise ticket} (at level $\gamma$) for $c$.
\end{definition}

\begin{lemma}[Stability of deterministic sampling]
\label{lem:ode_stability}
View DDIM sampling as solving a probability-flow ODE $\dot{\x}=f_\theta(\x,t,c)$ from $t{=}1$ to $t{=}0$ with terminal condition $\x(1)=\x_T$ and output $\x(0)=G_\theta(\x_T,c)$.
If $f_\theta(\cdot,t,c)$ is $L$-Lipschitz in $\x$ for all $t\in[0,1]$, then for any $\x_T,\x'_T$,
\begin{equation}
\lVert G_\theta(\x_T,c)-G_\theta(\x'_T,c)\rVert \le e^{L}\lVert \x_T-\x'_T\rVert\,.
\end{equation}
\end{lemma}
\begin{proof}
Let $\x(t)$ and $\x'(t)$ be the ODE solutions with terminal values $\x_T$ and $\x'_T$, and define $\Delta(t)=\x(t)-\x'(t)$.
By Lipschitzness, $\frac{d}{dt}\lVert \Delta(t)\rVert \le \lVert \dot{\Delta}(t)\rVert \le L\lVert \Delta(t)\rVert$.
Grönwall's inequality gives $\lVert \Delta(0)\rVert \le e^{L}\lVert \Delta(1)\rVert$, yielding the claim. \qed
\end{proof}

\begin{proposition}[Tickets form semantic regions]
\label{prop:ticket_region}
Assume $\mathcal{F}_\text{Motion}$ is $L_F$-Lipschitz and its output norm is bounded away from $0$ on the null-prompt generations of interest, so that $\Phi_{\varnothing}$ is well-defined and locally Lipschitz.
Then $\tilde{S}_c$ is locally Lipschitz in $\x_T$.
In particular, if $\tilde{S}_c(\x_T^\star)\ge \gamma$ for some $\x_T^\star$, there exists $r>0$ such that every $\x_T$ in the ball $\{\lVert \x_T-\x_T^\star\rVert\le r\}$ satisfies $\tilde{S}_c(\x_T)\ge \gamma/2$.
\end{proposition}
\begin{proof}
By Lem.~\ref{lem:ode_stability} (with $c{=}\varnothing$) and Lipschitzness of $\mathcal{F}_\text{Motion}$, the unnormalized feature map $\x_T\mapsto \mathcal{F}_\text{Motion}(G_\theta(\x_T,\varnothing))$ is locally Lipschitz.
Since normalization is locally Lipschitz away from $0$, $\Phi_{\varnothing}$ is locally Lipschitz, and
\begin{equation}
\big|\tilde{S}_c(\x_T)-\tilde{S}_c(\x'_T)\big|
=\big|(\Phi_{\varnothing}(\x_T)-\Phi_{\varnothing}(\x'_T))^\top\tau(c)\big|
\le \lVert \Phi_{\varnothing}(\x_T)-\Phi_{\varnothing}(\x'_T)\rVert\,.
\end{equation}
Choosing $r$ small enough so that $\big|\tilde{S}_c(\x_T)-\tilde{S}_c(\x_T^\star)\big|\le \gamma/2$ for all $\lVert \x_T-\x_T^\star\rVert\le r$ yields the result. \qed
\end{proof}

Prop.~\ref{prop:ticket_region} offers a concrete mechanism: once a seed yields a semantically aligned null-prompt motion, nearby seeds tend to yield \emph{similar} motions and thus inherit the semantic bias, which matches our empirical locality observations.

\begin{proposition}[Dictionary size and diminishing returns]
\label{prop:best_of_n}
Let $\{\x_T^{(i)}\}_{i=1}^N$ be i.i.d.\ samples from $\mathcal{N}(\mathbf{0},\mathbf{I})$.
Define $p_{c,\gamma}=\mathbb{P}(\x_T\in\mathcal{W}_{c,\gamma})$.
Then the probability that at least one sampled seed is a winning ticket is
\begin{equation}
\mathbb{P}\!\left(\max_{i\in[N]} \tilde{S}_c(\x_T^{(i)})\ge \gamma\right)
= 1-(1-p_{c,\gamma})^N\,,
\end{equation}
which increases with $N$ but saturates as $N$ grows.
\end{proposition}
\begin{proof}
By independence, the event that no seed lands in $\mathcal{W}_{c,\gamma}$ has probability $(1-p_{c,\gamma})^N$; take the complement. \qed
\end{proof}

\begin{proposition}[Why null-prompt retrieval transfers to text prompts]
\label{prop:null_to_text}
Assume the conditioning-induced feature shift is bounded in the retrieval space:
there exists $\delta(c)\ge 0$ such that for all $\x_T$, $\lVert \Phi_c(\x_T)-\Phi_{\varnothing}(\x_T)\rVert \le \delta(c)$.
Then for all $\x_T$, the conditional score satisfies
\begin{equation}
\big|S_c(\x_T)-\tilde{S}_c(\x_T)\big| \le \delta(c)\,.
\end{equation}
Consequently, any seed with $\tilde{S}_c(\x_T)\ge \gamma+\delta(c)$ also satisfies $S_c(\x_T)\ge \gamma$.
\end{proposition}
\begin{proof}
Since $\tau(c)$ is unit-norm,
\begin{equation}
\begin{split}
|S_c(\x_T)-\tilde{S}_c(\x_T)|
&=|(\Phi_c(\x_T)-\Phi_{\varnothing}(\x_T))^\top\tau(c)| \\
&\le \lVert \Phi_c(\x_T)-\Phi_{\varnothing}(\x_T)\rVert \le \delta(c)\,.
\end{split}
\end{equation}
The second claim follows immediately. \qed
\end{proof}

Prop.~\ref{prop:null_to_text} explains why our dictionary, built from null-prompt generations, can still yield strong initializations for text-conditioned sampling: if conditioning acts as a bounded perturbation in the retrieval embedding space, then high intrinsic (null) alignment implies high conditional alignment.
This assumption is also consistent with classifier-free guidance, where each denoising update explicitly mixes conditional and unconditional predictions, making the null trajectory a persistent component of the final motion.

\begin{proposition}[KL regularization as a trust region]
\label{prop:kl_trust_region}
Consider maximizing alignment under a moment-based Gaussian prior constraint:
\begin{equation}
\max_{\x_T}\ S_c(\x_T)\quad \text{s.t.}\quad
D_{KL}\!\left(\mathcal{N}(\mu(\x_T), \sigma^2(\x_T)) \| \mathcal{N}(0, 1)\right)\le \varepsilon,
\end{equation}
where $\mu(\x_T)$ and $\sigma^2(\x_T)$ are the empirical mean and variance of the noise tensor.
The objective in Eq.~\ref{eq:optimization} is the Lagrangian relaxation of this constrained problem (for some multiplier $\lambda$), encouraging alignment improvement while keeping the optimized seed close to the Gaussian prior.
\end{proposition}
\begin{proof}
This follows from standard Lagrangian relaxation: introducing $\lambda\ge 0$ yields the unconstrained objective
$-S_c(\x_T)+\lambda\,D_{KL}(\cdot)$ up to an additive constant, which matches Eq.~\ref{eq:optimization} after converting to a minimization problem. \qed
\end{proof}

\begin{proposition}[KL controls mean/variance drift]
\label{prop:kl_moment}
Let $\mu$ and $\sigma^2$ be the empirical mean and variance of a noise tensor.
Then
\begin{equation}
D_{KL}\!\left(\mathcal{N}(\mu, \sigma^2) \| \mathcal{N}(0, 1)\right)
= \frac{1}{2}\!\left(\mu^2 + \sigma^2 - 1 - \log \sigma^2\right)
\ge \frac{1}{2}\mu^2 + \frac{(\sigma^2-1)^2}{4(1+\sigma^2)}\,,
\end{equation}
so a small KL penalty forces both $\mu\!\approx\!0$ and $\sigma^2\!\approx\!1$.
\end{proposition}
\begin{proof}
The closed form is the standard KL divergence between univariate Gaussians.
For the bound, it suffices to show that for all $x>0$,
\begin{equation}
x-1-\log x \ge \frac{(x-1)^2}{2(1+x)}\,.
\end{equation}
Let $\phi(x)=x-1-\log x-\frac{(x-1)^2}{2(1+x)}$.
We have $\phi(1)=\phi'(1)=0$ and
\begin{equation}
\phi''(x)=\frac{1}{x^2}-\frac{4}{(1+x)^3}
=\frac{x^3-x^2+3x+1}{x^2(1+x)^3}\ge 0\,,
\end{equation}
so $\phi$ is convex and minimized at $x{=}1$, implying $\phi(x)\ge 0$.
Substituting $x=\sigma^2$ completes the proof. \qed
\end{proof}

\section{Implementation Details}
\label{sec:supp_setup}

\subsection{Base Models and Sampling}
\label{sec:supp_base}
We adopt MDM-50step~\cite{tevet2023human} and MotionLCM-V2~\cite{motionlcm} as base generators.
For MDM, we use DDIM sampling with 50 steps for both text and null prompts.
For MotionLCM, we adopt the 1-step inference setting.
Classifier-free guidance is applied in the same manner as in the original implementations.

\subsection{Noise Dictionary and Retrieval}
\label{sec:supp_dict}
To construct the noise dictionary used in Stage~1, we sample 10{,}000 noise tensors from the standard Gaussian distribution.
To reduce the retrieval cost, we create truncated versions of each noise at 4-frame intervals, and use motions truncated to \(40, 44, \dots, 196\) frames for Frame-Adjusted Retrieval (FAR).
For motion--text retrieval in both stages, we use TMR~\cite{petrovich2023tmr}, which encodes motion and text tokens into 256-dimensional feature vectors.
All experiments use the official pretrained TMR weights.

\subsection{Optimization Setup}
\label{sec:supp_opt}
Given the retrieved initial noise, we optimize it using the objective described in the main paper.
The balancing weight \(\lambda\) scales with the dimensionality of the noise space. On HumanML3D, we set \(\lambda{=}5000\) for MDM, which operates in the data space, and \(\lambda{=}250\) for MotionLCM, which uses a low-dimensional latent space. For the MTT benchmark, we increase the MDM weight to \(\lambda{=}50{,}000\) to account for the longer generated sequences and correspondingly higher-dimensional noise.
We use the Adam optimizer with a learning rate of \(0.01\) for MDM and \(0.05\) for MotionLCM, and perform up to 100 gradient steps per sample.
The optimization accepts either random Gaussian noise (WIN-O; Sec.~\ref{sec:supp_variants}) or a seed retrieved by WIN-FAR (WINRO; Table~\ref{tab:win_far_nulltext}) as input.

\subsection{LoRA-Based Noise Refiner}
\label{sec:supp_lora}

The LoRA-based noise refiner amortizes the test-time optimization of WINRO into a single forward pass.
Given an initial noise sample $\x_T$, the refiner $R_\phi$ predicts a residual $\Delta\x_T$ so that the refined seed $\hat{\x}_T = \x_T + \Delta\x_T$ lies in a region of noise space that yields higher text--motion alignment.

\paragraph{Architecture.}
We construct $R_\phi$ by attaching rank-64 LoRA modules~\cite{hu2022lora} with $\alpha{=}64$ (scaling factor $\alpha/r{=}1$) to a copy of the pretrained generator $G_\theta$.
Adapters are injected into all query, key, value, and output projections in every attention layer, all feed-forward layers, and the final output projection head.
To ensure that the network output is purely the learned residual $\Delta\x_T$ rather than a denoised prediction, we zero out the base weights and biases of the output head before training.
Combined with the standard zero-initialization of the LoRA output matrices~\cite{hu2022lora}, this guarantees $R_\phi(\x_T) = \mathbf{0}$ at initialization, so that training begins from the identity mapping in noise space.

\paragraph{Delta computation.}
The refiner performs a \emph{single} forward pass of the LoRA-augmented denoiser at a fixed anchor timestep $t_0$, set to the first step of the diffusion schedule (i.e.\ the maximum noise level):
\begin{equation}
  \Delta\x_T = R_\phi(\x_T,\, t_0,\, c),
  \label{eq:delta}
\end{equation}
where $c$ is the text conditioning.
No reverse-diffusion iterations are executed; the single-pass output is directly used as the residual.

\paragraph{Training.}
We freeze the base generator $G_\theta$ and optimize only the LoRA parameters $\phi$ end to end.
Each training iteration proceeds as follows:
(i)~sample $\x_T \sim \mathcal{N}(\mathbf{0}, \mathbf{I})$;
(ii)~compute $\Delta\x_T$ via Eq.~\eqref{eq:delta};
(iii)~run the \emph{full} reverse-diffusion process of the frozen $G_\theta$ from $\hat{\x}_T = \x_T + \Delta\x_T$ to obtain a generated motion $\x_0$;
(iv)~evaluate text--motion alignment via TMR~\cite{petrovich2023tmr}.
Gradients flow from the TMR similarity through the entire generation pipeline back to the LoRA parameters.
Because $\x_T$ is drawn from the same Gaussian prior that the retrieval dictionary is built from, no explicit supervision pairs $(\x_T^\dagger, \x_T^*)$ are required.

The training objective is:
\begin{equation}
  \mathcal{L}(\phi)
  = 1 - s(\mathcal{F}_\text{Motion}(\x_0),\, c)
    + \lambda \,\lVert \Delta\x_T \rVert_2^2,
  \label{eq:lora_loss}
\end{equation}
where $s(\cdot, c)$ is the cosine similarity defined in the main paper and $\lambda{=}0.5$ controls the strength of the $\ell_2$ regularization on the residual, following Eyring et al.~\cite{eyring2025noise}.

Table~\ref{tab:lora_hparams} summarizes the training hyperparameters.
Both configurations use AdamW with learning rate $1 \times 10^{-4}$ and gradient clipping at norm~$1.0$, trained for 25 epochs on the HumanML3D training split.
\begin{table}[h]
  \centering
  \caption{LoRA noise refiner training hyperparameters.}
  \label{tab:lora_hparams}
  \small
  \begin{tabular}{lcc}
    \toprule
                          & MDM   & MotionLCM \\
    \midrule
    LoRA rank $r$         & 64    & 64        \\
    LoRA $\alpha$         & 64    & 64        \\
    Batch size            & 32    & 64        \\
    Learning rate         & $10^{-4}$ & $10^{-4}$ \\
    Regularization $\lambda$ & 0.5 & 0.5       \\
    Gradient clip norm    & 1.0   & 1.0       \\
    Training Epochs       & 25    & 25        \\
    \bottomrule
  \end{tabular}
\end{table}

\paragraph{Inference.}
At test time, the refiner accepts either random Gaussian noise or the output of WIN-FAR retrieval as input, since both are drawn from the same prior.
When applied to random seeds without retrieval (WIN-LoRA; see Table~\ref{tab:winro_full_variants}), the refiner moves each seed toward a winning-ticket region while preserving per-prompt diversity.
In the WIN-RLoRA rows of Table~\ref{tab:win_far_nulltext}, we first retrieve a seed via WIN-FAR and then apply the refiner, which further improves text--motion alignment.

\subsection{Motion Stylization Setup}
\label{sec:supp_stylization}

As illustrated in Fig.~\ref{fig:styleoverview} (Sec.~\ref{sec:task_adapt}), the stylization pipeline extends the base WINRO framework.
We construct the style-specific noise dictionary $\mathcal{D}_s$ by applying DDIM inversion to motions from each style category of the 100STYLE dataset~\cite{mason2022local}. Because these entries originate from style motions rather than random Gaussian samples, they inherently encode stylistic attributes but offer limited text-semantic coverage.
Given a text prompt, we retrieve the best-matching entry from $\mathcal{D}_s$ using the same text-based retrieval as in the base framework (Eq.~\ref{eq:retrieval}).
The retrieved noise is then refined by jointly optimizing the semantic loss $\mathcal{L}_{sim}$ and a style loss that measures cosine distance in the style feature space between the generated motion and the reference style motion.

We adopt SMooDi~\cite{zhong2024smoodi}, built on MLD~\cite{chen2023executing}, as the stylization backbone.
In addition to the standard metrics, we report Style Recognition Accuracy (SRA), which measures how well the generated motion preserves the target style, and Foot Skating Rate, which quantifies physically implausible foot sliding artifacts.

\subsection{WINRO Pipeline Pseudocode}
\label{sec:supp_pseudocode}
Algorithm~\ref{alg:winro} summarizes the complete WINRO pipeline.
The dictionary is constructed once offline per backbone; at inference, only retrieval and optional refinement are executed.
The refinement loop backpropagates through the entire deterministic denoising chain of $G_\theta$; this is feasible because $G_\theta$ is frozen and all operations are differentiable under DDIM sampling.

\begin{algorithm}[t]
\caption{WINRO: Winning Noise Retrieval and Optimization}
\label{alg:winro}
\begin{algorithmic}[1]
\REQUIRE Frozen generator $G_\theta$, motion--text encoders $\mathcal{F}_\text{Motion}$, $\mathcal{F}_\text{Text}$, dictionary size $N$, text prompt $c$, target length $l$, KL weight $\lambda$, optimization steps $S$, learning rate $\alpha$
\ENSURE Refined initial noise $\x_T^*$

\STATE \textbf{--- Dictionary construction (offline, once per backbone) ---}
\FOR{$i = 1, \ldots, N$}
    \STATE Sample $\x_T^{(i)} \sim \mathcal{N}(\mathbf{0}, \mathbf{I})$
    \STATE Generate null-prompt motion: $\x_0^{(i)} = G_\theta(\x_T^{(i)}, \varnothing)$
    \FOR{each target length $l' \in \{m, m{+}\Delta, \ldots, M\}$}
        \STATE $\z_{l'}^{(i)} \leftarrow \mathcal{F}_\text{Motion}(\x_0^{(i)}[{:}l'])$
    \ENDFOR
    \STATE Store $\{(\z_{l'}^{(i)}, \x_T^{(i)})\}$ in dictionary $\mathcal{D}$
\ENDFOR

\STATE \textbf{--- Retrieval (per query) ---}
\STATE $\mathbf{t} \leftarrow \mathcal{F}_\text{Text}(c)$
\STATE $\x_T^{\dagger} \leftarrow \arg\max_{(\z, \x_T) \in \mathcal{D}} \cos(\z, \mathbf{t})$
\STATE $\x_T^{\dagger} \leftarrow \x_T^{\dagger} + \boldsymbol{\eta},\quad \boldsymbol{\eta} \sim \mathcal{N}(\mathbf{0}, \sigma^2 \mathbf{I})$ \hfill $\triangleright$ perturbation

\STATE \textbf{--- Refinement (per query) ---}
\STATE $\x_T \leftarrow \x_T^{\dagger}$
\FOR{$s = 1, \ldots, S$}
    \STATE $\x_0 \leftarrow G_\theta(\x_T, c)$ \hfill $\triangleright$ full deterministic denoising
    \STATE $\mathcal{L} \leftarrow (1 - \cos(\mathcal{F}_\text{Motion}(\x_0),\, \mathbf{t})) + \lambda \, D_\text{KL}(\mathcal{N}(\mu(\x_T), \sigma^2(\x_T)) \| \mathcal{N}(0,1))$
    \STATE $\x_T \leftarrow \x_T - \alpha \,\nabla_{\x_T} \mathcal{L}$ \hfill $\triangleright$ gradient through frozen $G_\theta$
\ENDFOR
\RETURN $\x_T^* \leftarrow \x_T$
\end{algorithmic}
\end{algorithm}

\section{Extended Variant Comparison}
\label{sec:supp_variants}

Table~\ref{tab:winro_full_variants} extends Table~\ref{tab:win_far_nulltext} with additional WINRO variants on HumanML3D, including optimization from random noise without retrieval (WIN-O), the LoRA-based noise refiner applied to random seeds without retrieval (WIN-LoRA; Sec.~\ref{sec:supp_lora}), and stochastic sampling strategies.
Each row is characterized by three design choices:
\textbf{Retrieval}---initial noise selection from the noise dictionary (Sec.~\ref{sec:stage_1}): random Gaussian (\,--\,), Top-1 retrieval, or Top-$k$ ($k{=}150$) retrieval;
\textbf{Refine}---noise refinement: none (\,--\,), KL-regularized optimization (Optim; Sec.~\ref{sec:stage_2}), or the LoRA-based noise refiner (LoRA; Sec.~\ref{sec:supp_lora});
\textbf{Perturbation}---Gaussian perturbation of the final noise ($\sigma{=}0.5$; Eq.~\ref{eq:perturbation}).
Top-$k$ retrieval and perturbation are stochastic strategies that recover per-prompt diversity; the main paper uses Top-1 retrieval with perturbation for WIN-FAR and omits perturbation for WINRO where KL regularization already preserves the prior.
Shaded rows correspond to configurations reported in Table~\ref{tab:win_far_nulltext}.

Table~\ref{tab:winro_full_variants} additionally includes \textbf{MultiModality (MModality)}~\cite{Guo_2022_CVPR_humanml3d}, which measures the diversity of motions generated from the \emph{same} text prompt; higher values indicate greater per-prompt variety.
Deterministic configurations that produce a single output per prompt have undefined MModality, shown as ``--''.
Results are averaged over 20 evaluation runs with 95\% confidence intervals.

\begin{table}[t]
    \centering
    \caption{Extended comparison of WINRO variants on HumanML3D.
    Retr.: initial noise selection from the noise dictionary---random Gaussian (\,--\,), Top-1, or Top-$k$ ($k{=}150$);
    Refine: noise refinement---none (\,--\,), KL-regularized optimization (Optim), or the LoRA-based noise refiner (LoRA);
    Pert.: Gaussian perturbation of the final noise ($\sigma{=}0.5$).
    \colorbox[gray]{0.93}{Shaded rows} correspond to Table~\ref{tab:win_far_nulltext}.}
    \label{tab:winro_full_variants}
    \setlength{\tabcolsep}{2pt}

    \vspace{2pt}
    \textbf{MDM-50step}\\[2pt]
    \resizebox{\linewidth}{!}{
    \begin{tabular}{ccc ccccccc}
        \toprule
        \textbf{Retr.} & \textbf{Refine} & \textbf{Pert.}
        & \textbf{FID} $\downarrow$ & \textbf{R-Top1} $\uparrow$ & \textbf{R-Top2} $\uparrow$ & \textbf{R-Top3} $\uparrow$
        & \textbf{Div.} $\rightarrow$ & \textbf{MMDist} $\downarrow$ & \textbf{MModality} $\uparrow$ \\
        \midrule
        \rowcolor[gray]{0.93} -- & -- & --
        & $0.438^{\scriptstyle\pm.009}$ & $0.469^{\scriptstyle\pm.002}$ & $0.661^{\scriptstyle\pm.002}$ & $0.763^{\scriptstyle\pm.002}$
        & $10.025^{\scriptstyle\pm.066}$ & $3.258^{\scriptstyle\pm.009}$ & $2.169^{\scriptstyle\pm.050}$ \\
        \rowcolor[gray]{0.93} Top-1 & -- & $\checkmark$
        & $0.248^{\scriptstyle\pm.009}$ & $0.512^{\scriptstyle\pm.003}$ & $0.706^{\scriptstyle\pm.003}$ & $0.804^{\scriptstyle\pm.002}$
        & $10.295^{\scriptstyle\pm.033}$ & $3.013^{\scriptstyle\pm.011}$ & $1.422^{\scriptstyle\pm.028}$ \\
        Top-$k$ & -- & $\checkmark$
        & $0.190^{\scriptstyle\pm.003}$ & $0.493^{\scriptstyle\pm.003}$ & $0.690^{\scriptstyle\pm.002}$ & $0.789^{\scriptstyle\pm.002}$
        & $9.652^{\scriptstyle\pm.035}$ & $3.104^{\scriptstyle\pm.006}$ & $\mathbf{2.302}^{\scriptstyle\pm.017}$ \\
        \cmidrule(lr){1-10}
        -- & Optim & --
        & $0.106^{\scriptstyle\pm.004}$ & $0.528^{\scriptstyle\pm.003}$ & $0.718^{\scriptstyle\pm.002}$ & $0.808^{\scriptstyle\pm.002}$
        & $9.638^{\scriptstyle\pm.041}$ & $2.935^{\scriptstyle\pm.006}$ & -- \\
        \rowcolor[gray]{0.93} Top-1 & Optim & --
        & $0.115^{\scriptstyle\pm.002}$ & $0.539^{\scriptstyle\pm.002}$ & $0.727^{\scriptstyle\pm.004}$ & $0.817^{\scriptstyle\pm.002}$
        & $9.657^{\scriptstyle\pm.026}$ & $2.883^{\scriptstyle\pm.010}$ & -- \\
        Top-1 & Optim & $\checkmark$
        & $\mathbf{0.093}^{\scriptstyle\pm.004}$ & $0.529^{\scriptstyle\pm.002}$ & $0.714^{\scriptstyle\pm.005}$ & $0.810^{\scriptstyle\pm.002}$
        & $9.418^{\scriptstyle\pm.029}$ & $2.933^{\scriptstyle\pm.009}$ & $1.171^{\scriptstyle\pm.039}$ \\
        \cmidrule(lr){1-10}
        -- & LoRA & --
        & $0.237^{\scriptstyle\pm.006}$ & $0.533^{\scriptstyle\pm.003}$ & $0.728^{\scriptstyle\pm.003}$ & $0.822^{\scriptstyle\pm.002}$
        & $9.885^{\scriptstyle\pm.060}$ & $2.931^{\scriptstyle\pm.008}$ & $1.620^{\scriptstyle\pm.043}$ \\
        Top-1 & LoRA & --
        & $0.238^{\scriptstyle\pm.005}$ & $\mathbf{0.550}^{\scriptstyle\pm.002}$ & $\mathbf{0.744}^{\scriptstyle\pm.003}$ & $\mathbf{0.834}^{\scriptstyle\pm.002}$
        & $10.303^{\scriptstyle\pm.022}$ & $\mathbf{2.846}^{\scriptstyle\pm.006}$ & -- \\
        \rowcolor[gray]{0.93} Top-1 & LoRA & $\checkmark$
        & $0.102^{\scriptstyle\pm.002}$ & $0.546^{\scriptstyle\pm.003}$ & $0.741^{\scriptstyle\pm.003}$ & $0.832^{\scriptstyle\pm.002}$
        & $9.577^{\scriptstyle\pm.025}$ & $\mathbf{2.846}^{\scriptstyle\pm.007}$ & $1.097^{\scriptstyle\pm.017}$ \\
        Top-$k$ & LoRA & $\checkmark$
        & $0.098^{\scriptstyle\pm.002}$ & $0.535^{\scriptstyle\pm.002}$ & $0.732^{\scriptstyle\pm.002}$ & $0.823^{\scriptstyle\pm.002}$
        & $\mathbf{9.528}^{\scriptstyle\pm.023}$ & $2.897^{\scriptstyle\pm.006}$ & $1.826^{\scriptstyle\pm.016}$ \\
        \bottomrule
    \end{tabular}}

    \vspace{6pt}
    \textbf{MotionLCM}\\[2pt]
    \resizebox{\linewidth}{!}{
    \begin{tabular}{ccc ccccccc}
        \toprule
        \textbf{Retr.} & \textbf{Refine} & \textbf{Pert.}
        & \textbf{FID} $\downarrow$ & \textbf{R-Top1} $\uparrow$ & \textbf{R-Top2} $\uparrow$ & \textbf{R-Top3} $\uparrow$
        & \textbf{Div.} $\rightarrow$ & \textbf{MMDist} $\downarrow$ & \textbf{MModality} $\uparrow$ \\
        \midrule
        \rowcolor[gray]{0.93} -- & -- & --
        & $0.093^{\scriptstyle\pm.004}$ & $0.548^{\scriptstyle\pm.003}$ & $0.741^{\scriptstyle\pm.003}$ & $0.835^{\scriptstyle\pm.002}$
        & $9.566^{\scriptstyle\pm.069}$ & $2.763^{\scriptstyle\pm.008}$ & $\mathbf{1.799}^{\scriptstyle\pm.062}$ \\
        \rowcolor[gray]{0.93} Top-1 & -- & $\checkmark$
        & $0.083^{\scriptstyle\pm.003}$ & $0.572^{\scriptstyle\pm.003}$ & $0.767^{\scriptstyle\pm.002}$ & $0.855^{\scriptstyle\pm.001}$
        & $9.474^{\scriptstyle\pm.028}$ & $2.674^{\scriptstyle\pm.006}$ & $1.051^{\scriptstyle\pm.040}$ \\
        Top-$k$ & -- & $\checkmark$
        & $0.085^{\scriptstyle\pm.003}$ & $0.552^{\scriptstyle\pm.002}$ & $0.751^{\scriptstyle\pm.002}$ & $0.842^{\scriptstyle\pm.002}$
        & $9.405^{\scriptstyle\pm.024}$ & $2.724^{\scriptstyle\pm.005}$ & $1.741^{\scriptstyle\pm.060}$ \\
        \cmidrule(lr){1-10}
        -- & Optim & --
        & $0.066^{\scriptstyle\pm.003}$ & $0.573^{\scriptstyle\pm.003}$ & $0.764^{\scriptstyle\pm.002}$ & $0.855^{\scriptstyle\pm.002}$
        & $9.176^{\scriptstyle\pm.022}$ & $2.628^{\scriptstyle\pm.004}$ & -- \\
        \rowcolor[gray]{0.93} Top-1 & Optim & --
        & $0.072^{\scriptstyle\pm.002}$ & $\mathbf{0.580}^{\scriptstyle\pm.002}$ & $\mathbf{0.775}^{\scriptstyle\pm.002}$ & $\mathbf{0.861}^{\scriptstyle\pm.001}$
        & $9.242^{\scriptstyle\pm.026}$ & $\mathbf{2.609}^{\scriptstyle\pm.005}$ & -- \\
        Top-1 & Optim & $\checkmark$
        & $\mathbf{0.061}^{\scriptstyle\pm.002}$ & $0.569^{\scriptstyle\pm.002}$ & $0.767^{\scriptstyle\pm.002}$ & $0.857^{\scriptstyle\pm.002}$
        & $9.447^{\scriptstyle\pm.026}$ & $2.641^{\scriptstyle\pm.005}$ & $1.108^{\scriptstyle\pm.045}$ \\
        \cmidrule(lr){1-10}
        -- & LoRA & --
        & $0.103^{\scriptstyle\pm.005}$ & $0.559^{\scriptstyle\pm.002}$ & $0.754^{\scriptstyle\pm.002}$ & $0.845^{\scriptstyle\pm.003}$
        & $9.637^{\scriptstyle\pm.092}$ & $2.714^{\scriptstyle\pm.008}$ & $1.603^{\scriptstyle\pm.066}$ \\
        Top-1 & LoRA & --
        & $0.114^{\scriptstyle\pm.003}$ & $0.575^{\scriptstyle\pm.002}$ & $0.774^{\scriptstyle\pm.001}$ & $\mathbf{0.861}^{\scriptstyle\pm.002}$
        & $10.076^{\scriptstyle\pm.018}$ & $2.640^{\scriptstyle\pm.005}$ & -- \\
        \rowcolor[gray]{0.93} Top-1 & LoRA & $\checkmark$
        & $0.083^{\scriptstyle\pm.002}$ & $\mathbf{0.580}^{\scriptstyle\pm.002}$ & $0.772^{\scriptstyle\pm.001}$ & $0.860^{\scriptstyle\pm.001}$
        & $\mathbf{9.550}^{\scriptstyle\pm.032}$ & $2.649^{\scriptstyle\pm.005}$ & $0.975^{\scriptstyle\pm.0042}$ \\
        Top-$k$ & LoRA & $\checkmark$
        & $0.078^{\scriptstyle\pm.002}$ & $0.562^{\scriptstyle\pm.002}$ & $0.760^{\scriptstyle\pm.002}$ & $0.849^{\scriptstyle\pm.002}$
        & $\mathbf{9.518}^{\scriptstyle\pm.027}$ & $2.681^{\scriptstyle\pm.004}$ & $1.617^{\scriptstyle\pm.068}$ \\
        \bottomrule
    \end{tabular}}
\end{table}

\paragraph{WIN-O: optimization without retrieval.}
WIN-O applies the test-time optimization objective (Eq.~\ref{eq:optimization}) directly to random Gaussian noise, bypassing the retrieval stage.
This already yields substantial gains over the base model (e.g., MDM FID: $0.438 \to 0.106$).
Adding Top-1 retrieval as initialization (WINRO) further improves R-Precision (MDM R-Top1: $0.528 \to 0.539$; MotionLCM: $0.573 \to 0.580$), consistent with the smoothness argument in Sec.~\ref{sec:theory_noise}: a retrieved seed lies closer to the target region, benefiting local refinement.

\paragraph{WIN-LoRA: LoRA refiner on random seeds.}
WIN-LoRA applies the LoRA refiner (Sec.~\ref{sec:supp_lora}) to random noise without retrieval, achieving competitive alignment while preserving non-trivial MModality, since different random seeds receive different corrections.
Compared with WIN-RLoRA (Table~\ref{tab:win_far_nulltext}), which combines Top-1 retrieval with the refiner, WIN-LoRA trades alignment (e.g., MDM R-Top1: $0.533$ vs.\ $0.550$) for per-prompt diversity.

\paragraph{Stochastic sampling.}
Adding perturbation to WIN-RLoRA recovers MModality while substantially reducing FID (e.g., MDM: $0.238 \to 0.102$).
Top-$k$ retrieval provides complementary diversity by sampling from a broader candidate set.

\section{Ablation Studies}
\label{sec:supp_ablation}

\subsection{Noise Dictionary Size}
\label{sec:supp_dictsize}

Table~\ref{tab:dictionary} ablates the noise dictionary size for WIN-FAR on MotionLCM, with Gaussian perturbation ($\sigma{=}0.5$) applied after retrieval. Results are averaged over 20 evaluation runs with 95\% confidence intervals.
Increasing the dictionary size steadily improves R-Precision and MMDist, indicating that a larger pool of candidate noises enables better text--motion matching.
FID improves from 100 to 10k entries and plateaus beyond that, suggesting diminishing returns from further expansion.
We use 10k as the default dictionary size throughout the paper, as it offers a favorable balance between alignment gains and dictionary construction cost.

\begin{table}[t]
  \centering
  \caption{Effect of noise dictionary size on MotionLCM with WIN-FAR (perturbation $\sigma{=}0.5$).}
  \label{tab:dictionary}
  \begin{tabular}{lccccccc}
    \toprule
    \textbf{Method} & \textbf{Size} & \textbf{FID} $\downarrow$ & \textbf{R-Top1} $\uparrow$ & \textbf{R-Top2} $\uparrow$ & \textbf{R-Top3} $\uparrow$ & \textbf{Div.} $\rightarrow$ & \textbf{MMDist} $\downarrow$ \\
    \midrule
    MotionLCM & -- & $0.093^{\scriptstyle\pm.004}$ & $0.548^{\scriptstyle\pm.003}$ & $0.741^{\scriptstyle\pm.003}$ & $0.835^{\scriptstyle\pm.002}$ & $9.566^{\scriptstyle\pm.069}$ & $2.763^{\scriptstyle\pm.008}$ \\
    \midrule
    w/ WIN-FAR & 100   & $0.094^{\scriptstyle\pm.006}$ & $0.559^{\scriptstyle\pm.002}$ & $0.754^{\scriptstyle\pm.002}$ & $0.845^{\scriptstyle\pm.002}$ & $9.333^{\scriptstyle\pm.019}$ & $2.730^{\scriptstyle\pm.009}$ \\
    w/ WIN-FAR & 1k    & $0.084^{\scriptstyle\pm.003}$ & $0.566^{\scriptstyle\pm.003}$ & $0.761^{\scriptstyle\pm.002}$ & $0.851^{\scriptstyle\pm.002}$ & $9.404^{\scriptstyle\pm.030}$ & $2.696^{\scriptstyle\pm.005}$ \\
    w/ WIN-FAR & 10k   & $0.083^{\scriptstyle\pm.003}$ & $0.572^{\scriptstyle\pm.003}$ & $0.767^{\scriptstyle\pm.002}$ & $0.855^{\scriptstyle\pm.001}$ & $9.474^{\scriptstyle\pm.028}$ & $2.674^{\scriptstyle\pm.006}$ \\
    w/ WIN-FAR & 100k  & $0.085^{\scriptstyle\pm.003}$ & $0.575^{\scriptstyle\pm.001}$ & $0.769^{\scriptstyle\pm.002}$ & $0.857^{\scriptstyle\pm.002}$ & $9.490^{\scriptstyle\pm.029}$ & $2.659^{\scriptstyle\pm.005}$ \\
    \bottomrule
  \end{tabular}
\end{table}

\subsection{Optimization Hyperparameters}
\label{sec:supp_opt_ablation}

Table~\ref{tab:ablation} ablates the refinement hyperparameters on MotionLCM, starting from the same WIN-FAR retrieval result.
The KL regularization weight $\lambda$ has the largest impact: without regularization ($\lambda{=}0$), the optimized noise departs from the Gaussian prior, degrading FID to $0.136$ despite comparable alignment. This confirms that the KL term in Eq.~\ref{eq:optimization} is essential for preserving generation quality during refinement. $\lambda{=}250$ yields the best balance between alignment and realism.
In contrast, performance is robust to the learning rate across an order of magnitude ($0.01$--$0.1$), and cosine similarity consistently outperforms Euclidean distance.
Based on these results, we adopt cosine similarity, $\lambda{=}250$, and lr$=$0.05 as defaults.

\begin{table}[t]
  \centering
  \caption{Ablation of WINRO refinement hyperparameters on MotionLCM. All runs start from the same WIN-FAR retrieval result and use 100 optimization steps. Each block varies one factor from the default setting (first row).}
  \label{tab:ablation}
  \resizebox{\linewidth}{!}{
  \begin{tabular}{lcccccc}
    \toprule
    \textbf{Method} & \textbf{FID} $\downarrow$ & \textbf{R-Top1} $\uparrow$ & \textbf{R-Top2} $\uparrow$ & \textbf{R-Top3} $\uparrow$ & \textbf{Div.} $\rightarrow$ & \textbf{MMDist} $\downarrow$ \\
    \midrule
    \textbf{Default} ($\lambda{=}250$, lr${=}0.05$, cos.) & \textbf{0.066} & \textbf{0.580} & 0.769 & 0.856 & 9.189 & 2.617 \\
    \midrule
    $\lambda=0$    & 0.136 & 0.549 & 0.741 & 0.837 & 9.833 & 2.717 \\
    $\lambda=2.5$  & 0.088 & 0.563 & 0.756 & 0.847 & 9.046 & 2.682 \\
    $\lambda=25$   & 0.070 & 0.567 & 0.762 & 0.850 & 9.108 & 2.657 \\
    $\lambda=2500$ & 0.083 & 0.577 & \textbf{0.774} & \textbf{0.863} & 9.242 & \textbf{2.616} \\
    \midrule
    lr$=$0.005 & 0.084 & 0.572 & \textbf{0.774} & \textbf{0.864} & 9.248 & \textbf{2.616} \\
    lr$=$0.5   & 0.073 & 0.575 & 0.768 & 0.854 & 9.242 & 2.624 \\
    \midrule
    Euclidean & 0.095 & 0.563 & 0.767 & 0.851 & 9.255 & 2.671 \\
    \bottomrule
  \end{tabular}}
\end{table}

\subsection{Retrieval Backbone}
\label{sec:supp_other_retrieval}

In the main experiments we use TMR~\cite{petrovich2023tmr} as the retrieval model.
To check whether our framework depends on a particular retrieval backbone, we also replace TMR with PartTMR~\cite{yu2025remogpt} and repeat the experiments on MotionLCM.
All other components, including the diffusion model, optimization setup, and evaluation protocol, are kept unchanged.

\begin{table}[t]
  \centering
  \caption{Effect of different retrieval backbones for WINRO on MotionLCM.}
  \label{tab:other_retrieval}
  \resizebox{\linewidth}{!}{
  \begin{tabular}{lccccccc}
    \toprule
    \textbf{Method} & \textbf{Backbone} & \textbf{FID} $\downarrow$ & \textbf{R-Top1} $\uparrow$ & \textbf{R-Top2} $\uparrow$ & \textbf{R-Top3} $\uparrow$ & \textbf{Div.} $\rightarrow$ & \textbf{MMDist} $\downarrow$ \\
    \midrule
    MotionLCM & --        & $0.093^{\scriptstyle\pm.004}$ & $0.548^{\scriptstyle\pm.003}$ & $0.741^{\scriptstyle\pm.003}$ & $0.835^{\scriptstyle\pm.002}$ & $9.566^{\scriptstyle\pm.069}$ & $2.763^{\scriptstyle\pm.008}$ \\
    \midrule
    w/ WIN-FAR & PartTMR   & $0.086^{\scriptstyle\pm.003}$ & $0.569^{\scriptstyle\pm.002}$ & $0.765^{\scriptstyle\pm.002}$ & $0.854^{\scriptstyle\pm.002}$ & $9.450^{\scriptstyle\pm.019}$ & $2.679^{\scriptstyle\pm.005}$ \\
    w/ WINRO   & PartTMR   & $0.070^{\scriptstyle\pm.002}$ & $0.579^{\scriptstyle\pm.002}$ & $0.777^{\scriptstyle\pm.002}$ & $0.861^{\scriptstyle\pm.002}$ & $9.210^{\scriptstyle\pm.017}$ & $2.601^{\scriptstyle\pm.005}$ \\
    w/ WIN-LoRA & PartTMR  & $0.103^{\scriptstyle\pm.003}$ & $0.563^{\scriptstyle\pm.004}$ & $0.761^{\scriptstyle\pm.003}$ & $0.851^{\scriptstyle\pm.002}$ & $9.587^{\scriptstyle\pm.066}$ & $2.699^{\scriptstyle\pm.007}$ \\
    \bottomrule
  \end{tabular}}
\end{table}

As shown in Table~\ref{tab:other_retrieval}, the PartTMR-based variant shows the same qualitative behavior as the original TMR-based one.
This indicates that our method does not rely on a specific retrieval backbone and can work with different motion--text retrieval models as long as they provide a shared feature space for text and motion.

\subsection{LoRA Rank and Injection Ablation}
\label{sec:supp_lora_ablation}

Table~\ref{tab:lora_ablation} ablates the LoRA rank $r$ and injection configuration of the WIN-LoRA noise refiner on MotionLCM.
Performance is robust across ranks from 16 to 128, with the default setting ($r{=}64$, full injection) offering a good balance.
Removing attention or feed-forward injections slightly reduces alignment, confirming that both pathways contribute to the learned noise correction.

\begin{table}[t]
    \centering
    \caption{WIN-LoRA ablation on MotionLCM varying rank and injection site.}
    \label{tab:lora_ablation}
    \setlength{\tabcolsep}{1.8pt}
    \resizebox{\linewidth}{!}{
        \begin{tabular}{llcccccccc}
            \toprule
            \textbf{Factor} & \textbf{Variant} & \textbf{FID}$\downarrow$ & \textbf{R-Top1}$\uparrow$ & \textbf{R-Top2}$\uparrow$ & \textbf{R-Top3}$\uparrow$ & \textbf{MMDist}$\downarrow$ & \textbf{Div.}$\rightarrow$ & \textbf{$\Delta$R-Top1} \\
            \midrule
            WIN-LoRA & full, r64 & 0.103 & 0.559 & 0.754 & 0.845 & 2.714 & 9.637 & -- \\
            \midrule
            Rank & r16 & \textbf{0.097} & 0.555 & 0.750 & 0.842 & 2.721 & 9.577 & $-$0.004 \\
            Rank & r32 & 0.104 & 0.553 & 0.752 & 0.843 & 2.721 & 9.623 & $-$0.006 \\
            Rank & r128 & 0.109 & \textbf{0.560} & \textbf{0.757} & \textbf{0.847} & \textbf{2.697} & 9.658 & +0.001 \\
            Injection & w/o Attn & 0.112 & 0.558 & 0.755 & 0.844 & 2.714 & 9.705 & $-$0.001 \\
            Injection & w/o Attn, FFN & 0.097 & 0.553 & 0.751 & 0.842 & 2.727 & 9.701 & $-$0.006 \\
            \bottomrule
        \end{tabular}
    }
\end{table}

\section{Additional Generalization Experiments}
\label{sec:supp_generalization}

\subsection{KIT-ML Dataset}
\label{sec:supp_kitml}

To verify that WINRO generalizes beyond HumanML3D, we evaluate on KIT-ML~\cite{Plappert2016kit} using MDM-50step with the same noise dictionary.
Table~\ref{tab:kitml} shows that both retrieval and refinement yield consistent improvements, confirming that winning noise tickets are not specific to a single dataset.

\begin{table}[t]
    \centering
    \caption{Generalization to KIT-ML using MDM-50step. Results are averaged over 20 evaluation runs with 95\% confidence intervals.}
    \label{tab:kitml}
    \resizebox{\linewidth}{!}{
    \begin{tabular}{lcccccc}
        \toprule
        \textbf{Method} & \textbf{FID} $\downarrow$ & \textbf{R-Top1} $\uparrow$ & \textbf{R-Top2} $\uparrow$ & \textbf{R-Top3} $\uparrow$ & \textbf{Div.} $\rightarrow$ & \textbf{MMDist} $\downarrow$ \\
        \midrule
        MDM-50step & $0.382_{\pm0.050}$ & $0.443_{\pm0.007}$ & $0.658_{\pm0.008}$ & $0.777_{\pm0.009}$ & $11.364_{\pm0.129}$ & $2.733_{\pm0.026}$ \\
        w/ WIN-FAR & $0.366_{\pm0.018}$ & $0.451_{\pm0.007}$ & $\mathbf{0.681}_{\pm0.006}$ & $0.789_{\pm0.004}$ & $11.246_{\pm0.029}$ & $2.714_{\pm0.014}$ \\
        w/ WINRO & $\mathbf{0.347}_{\pm0.025}$ & $\mathbf{0.454}_{\pm0.011}$ & $0.678_{\pm0.013}$ & $\mathbf{0.792}_{\pm0.013}$ & $11.403_{\pm0.157}$ & $\mathbf{2.654}_{\pm0.028}$ \\
        \bottomrule
    \end{tabular}}
\end{table}

\subsection{Rare-Prompt Analysis}
\label{sec:supp_rare}

To assess whether WINRO also benefits rare or difficult prompts, we split the HumanML3D test set by CLIP-based prompt rareness following ReMoDiffuse~\cite{zhang2023remodiffuse}: each caption is scored by $r(p){=}1{-}\max_i \cos(E_\text{CLIP}(p), E_\text{CLIP}(t_i))$ over training captions $t_i$, which is independent of TMR.
Table~\ref{tab:rare} reports MMDist on the rarest 20\% and tail 5\% subsets.
WINRO provides larger improvements on rarer prompts, suggesting that noise retrieval and refinement are particularly beneficial when the base model has fewer relevant training examples to draw upon.
Note that this remains an in-distribution evaluation; similar to other optimization methods, WINRO cannot create motions outside the frozen backbone's learned support.

\begin{table}[t]
    \centering
    \caption{MMDist ($\downarrow$) on rare-prompt subsets of HumanML3D.}
    \label{tab:rare}
    \begin{tabular}{lcc}
        \toprule
        Method & Rare-20\% & Tail-5\% \\
        \midrule
        MDM-50step & 4.06 & 4.17 \\
        w/ WIN-FAR & 3.79 & 3.82 \\
        w/ WINRO   & 3.34 & 3.39 \\
        \bottomrule
    \end{tabular}
\end{table}

\section{Additional Qualitative Examples}
\label{sec:supp_qual}

We provide additional qualitative examples that complement the quantitative results and show how different initial noise settings affect the generated motions.

\begin{figure}[t]
    \centering
    \includegraphics[width=1\linewidth]{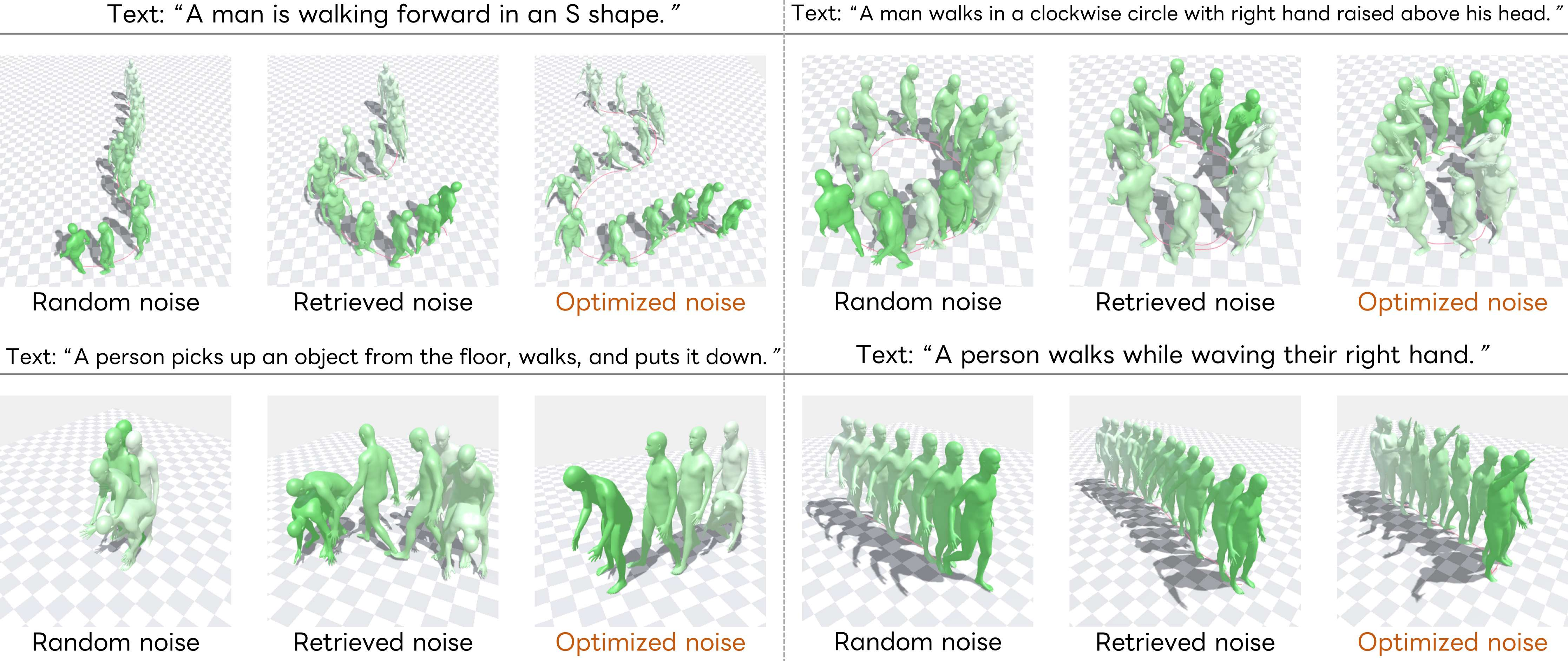}
    \caption{Additional comparison of motions generated with random, retrieved, and refined initial noise for each text prompt, using MDM-50step.}
    \label{fig:qualitative_additional}
\end{figure}

Fig.~\ref{fig:qualitative_additional} presents additional text-to-motion examples extending Fig.~\ref{fig:basecomparison}, comparing motions generated from random noise, retrieved noise (WIN-FAR), and optimized noise (WINRO) under MDM-50step.
Retrieved noise already produces motions closer to the text descriptions, and refinement further improves faithfulness.

\begin{figure}[t]
    \centering
    \includegraphics[width=1\linewidth]{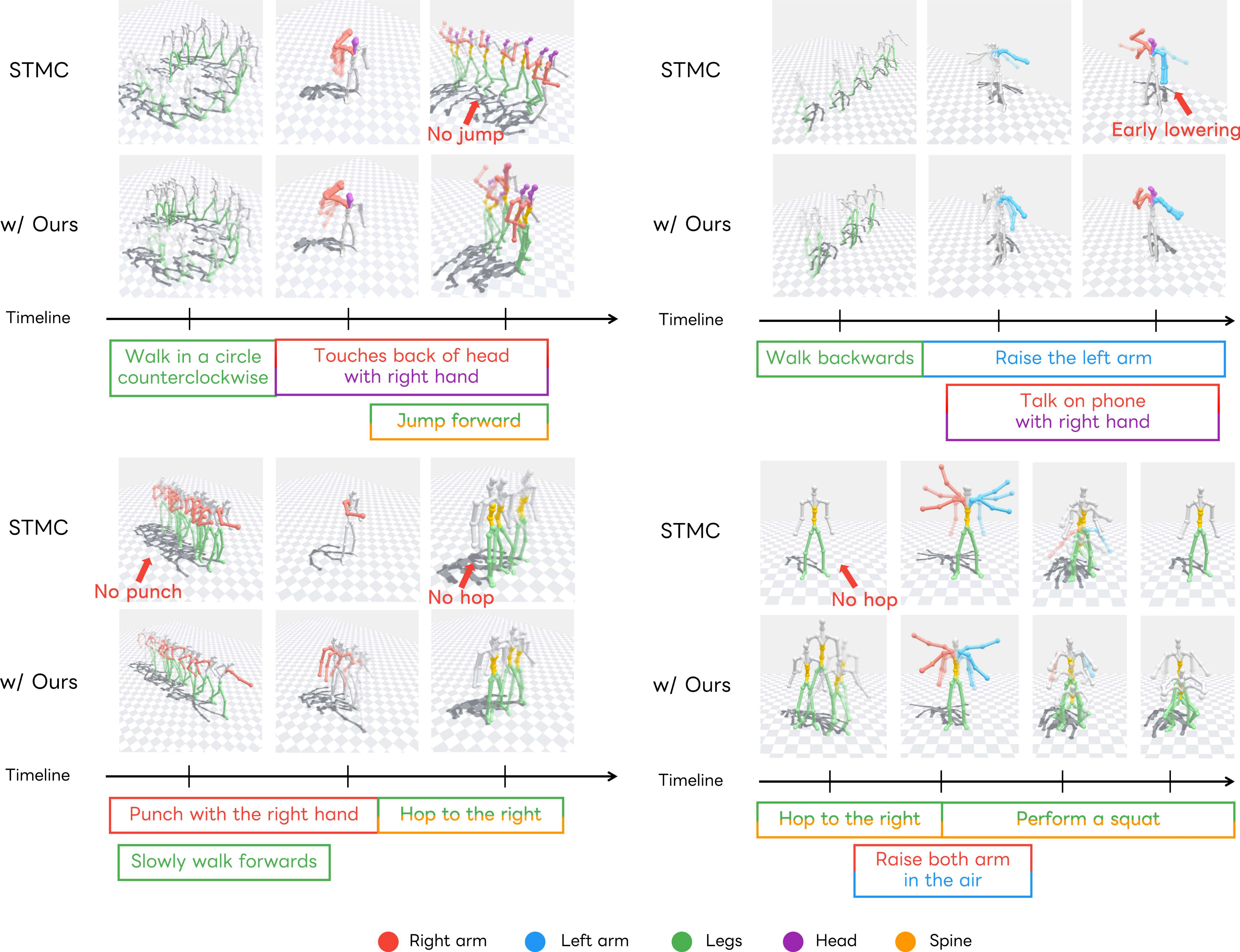}
    \caption{Additional compositional generation results on the multi-track timeline benchmark. STMC~\cite{petrovich2024stmc} often fails to realize individual action segments (\eg, missing jump, punch, or hop), while applying our noise refinement (w/ Ours) improves faithfulness to each segment's text prompt.}
    \label{fig:stmc_additional}
\end{figure}

Fig.~\ref{fig:stmc_additional} shows additional compositional generation examples using STMC~\cite{petrovich2024stmc} with and without our noise refinement, extending the comparison in Fig.~\ref{fig:stmc_qualitative}.
Across four multi-track timelines with overlapping and sequential action prompts, STMC frequently drops or truncates individual actions, whereas applying our method consistently improves segment-level faithfulness.

For animated visualizations of all results, we strongly encourage readers to view the supplementary video and the interactive motion viewer included in the supplementary materials.

\end{document}